\RequirePackage{iftex}
\ifpdftex
\fi
\documentclass[sigconf,nonacm]{acmart}

\AtBeginDocument{%
  }

\setcopyright{none}
\copyrightyear{2026}
\acmYear{2026}
\acmDOI{}
\acmBooktitle{}
\acmISBN{}

\microtypecontext{spacing=nonfrench}

\usepackage{bm}
\usepackage{mathrsfs}
\usepackage{bbm}
\usepackage{amsmath}

\usepackage{amsthm}

\usepackage{algorithmic}
\usepackage{algorithm}

\usepackage{array}
\usepackage[caption=false,font=normalsize,labelfont=sf,textfont=sf]{subfig}
\usepackage{textcomp}
\usepackage{stfloats}
\usepackage{verbatim}
\usepackage{graphicx}
\graphicspath{{figures/}}

\definecolor{myorange}{RGB}{255,128,0}
\definecolor{myblue}{RGB}{0,0,255}
\definecolor{myred}{RGB}{255,0,0}
\definecolor{mypurple}{RGB}{147,112,219}
\definecolor{bestcol}{RGB}{243,200,70}   
\definecolor{secondcol}{RGB}{247,239,174}
\newcommand{\legendbox}[2]{\colorbox{#1}{\strut\textcolor{black}{#2}}}
\newcommand{\best}[1]{\cellcolor{bestcol}{\color{black}{#1}}}     
\newcommand{\second}[1]{\cellcolor{secondcol}{\color{black}{#1}}} 
\definecolor{customcolor1}{RGB}{156, 74, 9}
\definecolor{customcolor2}{RGB}{234, 112, 13}
\definecolor{customcolor3}{RGB}{249, 196, 153}

\usepackage{afterpage}
\usepackage{multirow}
\usepackage{makecell}
\usepackage{rotating} 
\usepackage{booktabs}
\usepackage{diagbox}
\usepackage{colortbl}
\usepackage{tabularx}
\usepackage{hhline}

\usepackage{enumitem}
\theoremstyle{plain}
\newtheorem{theorem}{Theorem}
\newtheorem{lemma}{Lemma}

\theoremstyle{remark}

\usepackage{xspace}

\begin{document}

\title{TaFD: Threat-Aware Frequency Decoupling for Adversarial~Robustness against Heterogeneous Attacks}

\author{Mengda Xie}
\affiliation{%
  \institution{School of Computer Science and Cyber Engineering, Guangzhou University}
  \city{Guangzhou}
  \state{Guangdong}
  \postcode{510006}
  \country{China}}

\author{Yiling He}
\affiliation{%
  \institution{Information Security Research Group, University College London}
  \city{London}
  \postcode{NW1 2AE}
  \country{United Kingdom}}

\author{Meie Fang}
\authornote{Corresponding author: Meie Fang (fme@gzhu.edu.cn).}
\email{fme@gzhu.edu.cn}
\affiliation{%
  \institution{School of Computer Science and Cyber Engineering, Guangzhou University}
  \city{Guangzhou}
  \state{Guangdong}
  \postcode{510006}
  \country{China}}

\begin{abstract}
Multi-threat robustness remains a fundamental challenge in deep learning. Although joint adversarial training (JAT) is widely adopted, it suffers from negative transfer under heterogeneous threats, particularly between $\ell_p$-bounded and semantic attacks. Through first-order gradient analysis, we formalize this as gradient incompatibility and theoretically establish the necessity of decoupled optimization. We further reveal that these conflicting threats exhibit separable spectral characteristics in the frequency domain. Motivated by this observation, we propose Threat-aware Frequency Decoupling (TaFD), a two-stage defense framework that reformulates JAT as a frequency-domain divide-and-conquer paradigm. TaFD first discovers latent threat domains via unsupervised clustering of attack spectral prototypes and trains a lightweight classifier for inference-time threat domain identification. Conditioned on the prediction, TaFD employs a Frequency-Conditional Convolution that learns threat-domain-specific spectral masks and routes each sample to the corresponding expert, enforcing structural parameter separation and alleviating optimization conflicts.

We validate TaFD on three representative image-classification benchmarks (CIFAR-10, CIFAR-100, and Tiny-ImageNet) and on two representative architectures (the convolutional ResNet and the hybrid-transformer MobileViT). Extensive results demonstrate that TaFD achieves more balanced robustness against heterogeneous attacks than existing JAT and frequency-domain baselines, improving average robust accuracy by approximately 11\% over the strongest baseline while maintaining leading clean accuracy.
\end{abstract}



\maketitle

\section{Introduction}\label{sec1}

Deep neural networks (DNNs) achieve state-of-the-art performance on standard vision benchmarks and are increasingly used in safety-critical applications such as autonomous driving~\cite{DBLP:journals/ftcgv/JanaiGBG20} and medical imaging~\cite{DBLP:journals/mia/LitjensKBSCGLGS17}. Nevertheless, they remain vulnerable to adversarial examples, i.e., maliciously crafted inputs designed to induce misclassification through imperceptible perturbations or semantic manipulations~\cite{DBLP:journals/corr/SzegedyZSBEGF13, DBLP:journals/tcsv/XieHQF25}. This vulnerability poses a fundamental risk to system security and reliability, motivating robust, scalable defenses for trustworthy deployment.

\begin{figure}[!t]
\includegraphics[width=1\linewidth]{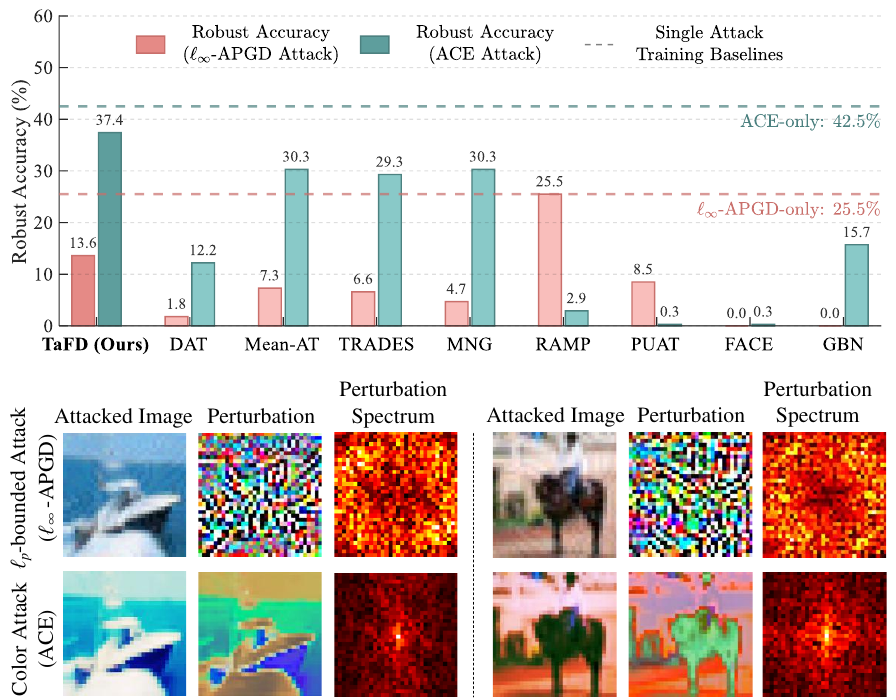}
\Description{Line charts illustrating trade-offs in joint adversarial training and images showing spectral signatures of heterogeneous attacks.}
\caption{\textbf{Trade-offs in joint adversarial training (Top) and spectral signatures of heterogeneous attacks (Bottom).} \textbf{Top:} Robust accuracies on CIFAR-100 under $\ell_\infty$-APGD and ACE. Dashed lines indicate single-attack training baselines. Existing methods exhibit pronounced trade-offs, whereas our method TaFD achieves the most balanced robustness. \textbf{Bottom:} Adversarial perturbations and their frequency spectra for two representative attack types. $\ell_p$-bounded attacks concentrate energy at higher spatial frequencies, while ACE induces globally consistent low-frequency chromatic shifts.}
\label{fig:tradeoff}
\end{figure}

Adversarial training constitutes the standard defense under a single threat model. However, real-world deployment demands multi-threat robustness, requiring models to maintain performance across a union of distinct attack types. To achieve this, joint adversarial training (JAT) is widely adopted. Yet, JAT is frequently hindered by negative transfer, wherein improving robustness to one threat degrades robustness to another. This trade-off is particularly acute when the union spans heterogeneous threats, such as $\ell_p$-bounded perturbations and semantic threats (e.g., color shifts or geometric transformations). This conflict arises because heterogeneous threats can impose incompatible optimization constraints on the shared model parameters.
Existing JAT paradigms address this conflict through loss-level aggregation or threat-coverage expansion, but neither resolves the underlying incompatibility. Aggregation-based approaches~\cite{tramer2019adversarial, maini2020adversarial, jiang2024ramp} combine multiple attacks via per-sample worst-case selection, averaging, or gradient projection, compressing heterogeneous gradients into a single update direction without reconciling the conflicting objectives. GBN~\cite{liu2024towards} moves beyond pure aggregation by introducing perturbation-specific batch normalization branches, yet the shared convolutional parameters remain subject to the same gradient conflict. Coverage-oriented approaches~\cite{DBLP:conf/iclr/Laidlaw0F21, poursaeed2021robustness, DBLP:journals/pami/ZhangYSY24} broaden the threat set without modeling the conflict structure, thereby likewise failing to mitigate negative transfer.

\begin{figure}[!t]
\centering
\includegraphics[width=1\linewidth]{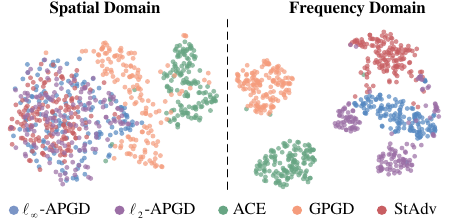}
\Description{t-SNE scatter plots comparing spatial and frequency domain feature representations of different adversarial attacks.}
\caption{\textbf{Spatial vs.\ frequency domain separability on CIFAR-100.} t-SNE visualization of perturbation features from five representative attacks ($\ell_\infty/\ell_2$-APGD~\cite{DBLP:conf/icml/Croce020a}, ACE~\cite{DBLP:journals/tifs/00010L23}, GPGD~\cite{DBLP:journals/pr/XiaoYFWL25}, StAdv~\cite{DBLP:conf/iclr/XiaoZ0HLS18}). \textbf{Left (Spatial):} Attacks are largely entangled with no clear cluster boundaries. \textbf{Right (Frequency):} All five attacks form well-separated clusters, demonstrating the frequency domain's superior discriminability.}
\label{fig:tsne_visualization}
\end{figure}

To formalize this conflict, we develop a first-order gradient analysis that characterizes negative transfer as a form of gradient incompatibility between heterogeneous attacks. We contrast two canonical threat classes targeting distinct perturbation patterns: $\ell_p$-bounded perturbations, which produce local, fine-grained pixel-level noise; and color transforms, which induce global, structured appearance shifts. Under the stated assumptions, the analysis shows that enforcing invariance to one threat increases the minimal gradient norm required for the margin, thereby raising the other threat's first-order robust-risk coefficient and making a unified objective prone to conflict.
Empirical results on CIFAR-100 (Fig.~\ref{fig:tradeoff}, Top), where these two threat classes are instantiated by $\ell_\infty$-APGD~\cite{DBLP:conf/icml/Croce020a} and ACE~\cite{DBLP:journals/tifs/00010L23} respectively, support this analysis. RAMP matches the $\ell_\infty$-APGD-only baseline at 25.5\% but ACE accuracy collapses to 2.9\%. Conversely, MNG preserves 30.3\% ACE accuracy but $\ell_\infty$-APGD drops to 4.7\%, over 20 percentage points below the single-attack baseline. These trade-offs empirically align with the gradient incompatibility formalized by our analysis, raising a natural question: can these conflicting objectives be disentangled in an alternative representation space?

Since $\ell_p$-bounded perturbations and color transforms operate at fundamentally different spatial scales, the Fourier transform naturally maps them to separable frequency bands, enabling their disentanglement. Indeed, as visualized in Fig.~\ref{fig:tradeoff} (Bottom), $\ell_p$-bounded attacks concentrate perturbation energy at higher spatial frequencies, while ACE induces globally consistent low-frequency chromatic shifts.
More broadly, since diverse attack paradigms, including geometric deformations, chromatic shifts, pixel-level noise, and on-manifold edits, manipulate images with characteristically different spatial structures, we hypothesize that this spectral separability extends beyond the canonical pair. To verify this, we visualize perturbation features from five representative attacks spanning $\ell_p$-bounded, color, spatial, and on-manifold threats. As shown in Fig.~\ref{fig:tsne_visualization}, frequency-domain t-SNE embeddings yield well-separated, compact clusters for all five attack types, whereas their spatial-domain counterparts are largely entangled with no discernible boundaries across most threat categories. This contrast suggests the frequency domain's substantially greater discriminative power, providing a principled foundation for threat-aware decoupling of robustness objectives that are otherwise entangled in the pixel domain.

Motivated by the spectral insight above, we propose Threat-aware Frequency Decoupling~(TaFD), a unified defense framework that reformulates JAT as a frequency-domain divide-and-conquer paradigm. Concretely, TaFD instantiates a \textbf{Diagnosis--Dispatch} architecture. The \textbf{diagnosis stage} dynamically discovers latent threat domains, providing a conditioning signal for the dispatch stage's subsequent frequency-domain processing. To achieve this, we derive spectral prototypes of known attacks and cluster them into latent threat domains, which are then used to train a lightweight classifier for threat-domain identification during inference.
The \textbf{dispatch stage} employs the predicted threat-domain label as an explicit conditioning signal for frequency-domain decoupling. It is implemented by a novel Frequency-Conditional Convolution~(FC-Conv) layer. FC-Conv jointly learns threat-domain-specialized spectral masks to isolate threat-relevant spectral regions and applies sample-level hard routing of the masked features to the corresponding expert, with gradient accessibility ensured via BPDA~\cite{DBLP:conf/icml/AthalyeC018} during adversarial evaluation, thereby enforcing frequency-grounded structural parameter separation. This reduces reliance on a single pixel-domain gradient reconciliation path and substantially alleviates optimization conflicts across heterogeneous threats. As shown in Fig.~\ref{fig:tradeoff} (Top), TaFD mitigates negative transfer, achieving the most balanced robustness with 13.6\% on $\ell_\infty$-APGD and 37.4\% on ACE.

\par Our contributions can be summarized as follows:
\begin{itemize}
    \item We identify negative transfer in JAT under heterogeneous threats and introduce a first-order gradient analysis that formalizes this conflict as gradient incompatibility, theoretically motivating decoupled optimization.
    \item We empirically demonstrate that the frequency domain serves as a natural representation space for disentangling heterogeneous threats, with diverse attack paradigms exhibiting distinct and well-separated spectral signatures.
    \item We propose TaFD, a Diagnosis--Dispatch framework that decouples heterogeneous defense in the frequency domain via latent threat-domain discovery and threat-specialized spectral expert routing, effectively mitigating optimization conflicts across heterogeneous threats.
    \item Extensive experiments on three benchmarks and two architectures show that TaFD achieves more balanced robustness against heterogeneous attacks than existing JAT and frequency-domain baselines, improving average robust accuracy by up to 11.3 percentage points over the strongest competitor while maintaining competitive clean accuracy.
\end{itemize}

\section{Related Work}\label{sec2}

\subsection{Adversarial Attacks}\label{sec2.1}
Adversarial attacks deceive DNNs via imperceptible or plausible perturbations. \(\ell_p\)-bounded attacks exploit model oversensitivity to local features,
evolving from L-BFGS~\cite{DBLP:journals/corr/SzegedyZSBEGF13} and
C\&W~\cite{DBLP:conf/sp/Carlini017} to efficient first-order methods like
FGSM~\cite{DBLP:journals/corr/GoodfellowSS14}, iterative
variants~\cite{DBLP:conf/iclr/KurakinGB17a}, and
PGD~\cite{DBLP:conf/iclr/MadryMSTV18}. These perturbations primarily alter
high-frequency textures rather than global structure.

In contrast to \(\ell_p\)-bounded attacks, semantic attacks apply structured, meaning-preserving transformations that yield macroscopic yet plausible modifications without tight norm constraints. This category encompasses three principal threat types. Spatial attacks modify image geometry via optimized rotations and translations~\cite{DBLP:conf/icml/EngstromTTSM19} or flow-based
deformations~\cite{DBLP:conf/iclr/XiaoZ0HLS18}. Parameterized attacks optimize latent variables to generate on-manifold perturbations. Early work in GAN latent space~\cite{zhao2017generating, DBLP:conf/eccv/QiuXYYLL20} has evolved into diffusion-based methods that inject targets during reverse sampling (AdvDiff)~\cite{DBLP:conf/eccv/DaiLX24} or optimize along the generative manifold (ACA)~\cite{DBLP:conf/nips/ChenLWJDZ23}. These on-manifold attacks may bypass standard adversarial training, leading to higher success rates and excess risk~\cite{DBLP:journals/pr/XiaoYFWL25}. Color transformations constitute a prominent and extensively studied subclass, incurring large \(\ell_p\) distances while capturing realistic photometric variations. Hosseini and Poovendran~\cite{DBLP:conf/cvpr/HosseiniP18} first demonstrated adversarial vulnerability to hue/saturation shifts in HSV space. Building on this direction, subsequent works developed colorization-based
optimization (cAdv)~\cite{DBLP:conf/iclr/BhattadCLLF20}, selective CIELAB modifications (ColorFool)~\cite{DBLP:conf/cvpr/ShamsabadiSC20}, and parameterized global color mapping (ReColorAdv)~\cite{DBLP:conf/nips/LaidlawF19}. Further extensions include ACE~\cite{DBLP:journals/tifs/00010L23}, ALA~\cite{DBLP:conf/mm/0001S0JZFLP23}, and RetouchUAA~\cite{DBLP:journals/tcsv/XieHQF25}.

The divergence between \(\ell_p\)-bounded and semantic attacks highlights a key challenge for unified robustness. These paradigms impose conflicting optimization objectives: robustness to local, fine-grained \(\ell_p\) perturbations versus invariance to macroscopic, semantic transformations. Resolving this optimization conflict is crucial for unified defenses.

\subsection{Adversarial Defenses}\label{sec2.2}
We first review defenses designed for individual threat types, then survey methods addressing multiple threats, and finally discuss frequency-domain approaches.

\subsubsection{Single-Attack Defenses}\label{sec2.2.1}

\begin{sloppypar}
Traditional adversarial defense primarily targets single \(\ell_p\)-bounded threats. Adversarial training (AT)~\cite{DBLP:conf/iclr/MadryMSTV18} is the leading paradigm, reformulating risk minimization as a min--max problem to induce robust features. Yet, AT is computationally expensive and often degrades clean accuracy. To address this, TRADES~\cite{DBLP:conf/icml/ZhangYJXGJ19} balances clean and robust errors via KL-regularization. Free AT~\cite{DBLP:conf/nips/ShafahiNG0DSDTG19} accelerates training by updating parameters and perturbations simultaneously, while Fast AT~\cite{DBLP:conf/iclr/WongRK20} uses random initialization with FGSM for low-cost robustness. However, \(\ell_p\) robustness rarely transfers to heterogeneous threats, necessitating joint defenses.
\end{sloppypar}

\subsubsection{Unified Defense Against Multiple Threats}\label{sec2.2.2}

To overcome the tendency of single-threat defenses to overspecialize, JAT aims to achieve multi-threat robustness—maintaining performance across a diverse union of attacks. Prominent aggregation-based methods began with simple loss aggregation, such as selecting the highest-loss example (MAX) or averaging losses (AVG)~\cite{tramer2019adversarial}. Subsequent works introduced geometry-aware mechanisms, such as MSD~\cite{maini2020adversarial}, which selects the max-loss ascent direction, and E-AT~\cite{croce2022adversarial}, which exploits \(\ell_1\)/\(\ell_{\infty}\) geometric relations. Recent feature-level defenses include GBN~\cite{liu2024towards}, which employs multi-branch batch normalization with a gated sub-network to identify perturbation types and normalize features accordingly, and RAMP~\cite{jiang2024ramp}, which balances multi-norm robustness via gradient
projection and logit pairing.
In parallel, coverage-oriented methods generate more diverse attacks. PAT~\cite{DBLP:conf/iclr/Laidlaw0F21} uses LPIPS for perceptual alignment. Semantic-space AT~\cite{poursaeed2021robustness} manipulates StyleGAN's latent variables to create multi-dimensional semantic variations. PUAT~\cite{DBLP:journals/pami/ZhangYSY24} employs distribution alignment against unconstrained adversarial examples (UAEs).

In summary, existing methods struggle to jointly address heterogeneous \(\ell_p\) and semantic threats. Aggregation schemes combine incompatible objectives, while coverage schemes overlook this incompatibility. Both hinder the resolution of negative transfer arising from gradient incompatibility. Explicitly modeling and reconciling these conflicting objectives at the optimization level is a promising direction for unified robustness.

\subsubsection{Defenses from a Frequency-Domain Perspective}\label{sec2.2.3}

Frequency-domain defenses exploit the distinct spectral characteristics of adversarial perturbations. Early work explored JPEG-based preprocessing~\cite{DBLP:conf/iclr/GuoRCM18, liu2019feature} to suppress adversarial energy while preserving task-critical features. Subsequent approaches manipulate spectral representations directly. Zhang et al.~\cite{zhang2019adversarial} introduced differentiable high-frequency attenuation, while the FACE framework~\cite{niu2023defense} applies band-specific DWT compression. A separate line distinguishes amplitude from phase: DAT~\cite{DBLP:conf/nips/LiLW0024} mixes training amplitudes with generative adversarial amplitudes, compelling the model to rely on robust phase information.
However, despite their simplicity, early preprocessing-based methods~\cite{DBLP:conf/iclr/GuoRCM18, zhang2019adversarial, niu2023defense} are sensitive to hyperparameters and often degrade clean accuracy, while training-integrated approaches like DAT~\cite{DBLP:conf/nips/LiLW0024} introduce new complexities. Critically, existing methods predominantly target single \(\ell_p\)-bounded threats. The integration of spectral biases into unified heterogeneous defense frameworks remains an open challenge. In contrast, our work leverages the frequency domain as a structured representation to actively decouple conflicting optimization objectives and mitigate negative transfer in heterogeneous JAT.

\section{Negative Transfer in JAT: A First-Order Analysis}
\label{sec:theory-neg-transfer-en}

This section presents a first-order gradient analysis that formalizes a conflict between two canonical heterogeneous threats: $\ell_p$-bounded perturbations and color transforms. Formalizing this incompatibility helps explain the negative transfer phenomenon in JAT and theoretically motivates decoupled optimization, underpinning the design of the TaFD defense framework.

\subsection{Setting and Preliminaries}
We consider random pairs $(\bm{x},y)\sim \mathcal{D}$, where $\mathcal{D}$ is a data distribution over the input space $\mathbb{R}^d$, which is the $d$-dimensional real vector space, and the label space $\mathcal{Y}$. We denote the loss function by $\mathcal{L}(\bm{\theta};\bm{x},y)$ (e.g., cross-entropy), where $\bm{\theta}$ denotes the model
parameters, and define the input gradient $\bm{g}(\bm{x})$ as $\nabla_{\bm{x}} \mathcal{L}(\bm{\theta};\bm{x},y)$. We fix the norm order $p\in(1,\infty)$ and denote by $q$ its H\"older conjugate, satisfying $\frac{1}{p} + \frac{1}{q} = 1$.
Throughout this section, $\langle \cdot,\cdot\rangle$ denotes the standard inner product on $\mathbb{R}^d$, and $\sup$ and $\inf$ denote the supremum and infimum.

\noindent\textbf{Linearized color subspace.}
We consider the family $\{T_{\bm{\gamma}}\}_{\bm{\gamma}\in\mathbb{R}^m}$ of differentiable color transforms on $\mathbb{R}^d$, normalized by $T_{\bm{0}}(\bm{x})=\bm{x}$ so that $\bm{\gamma}=\bm{0}$ serves as the reference point for local analysis. The local color directions are defined as:
\begin{equation*}
\bm{v}_i(\bm{x}):=\left.\frac{\partial\,T_{\bm{\gamma}}(\bm{x})}{\partial \gamma_i}\right|_{\bm{\gamma}=\bm{0}}\in\mathbb{R}^d,\qquad i=1,\dots,m.
\end{equation*}
The linearized color subspace $V_c(\bm{x})$ is defined as:
\begin{equation*}
V_c(\bm{x}):=\operatorname{span}\,\{\,\bm{v}_1(\bm{x}),\dots,\bm{v}_m(\bm{x})\,\}\subset\mathbb{R}^d.
\end{equation*}
Intuitively, $V_c(\bm{x})$ represents the tangent space of the color transformation manifold at $\bm{x}$, capturing all possible directions of local color change. We focus on this linearized subspace to facilitate a tractable first-order analysis.

\noindent\textbf{Robust risks.}
We define the $\ell_p$ robust risk $R_p(\bm{\theta},\varepsilon)$ and the linearized color robust risk $R_c(\bm{\theta},\eta)$, representing the expected worst-case loss under respective perturbations of magnitude $\varepsilon>0$ and $\eta>0$:
\begin{align}
R_p(\bm{\theta},\varepsilon) &:= \mathbb{E}_{(\bm{x},y)\sim\mathcal{D}}\Big[\,\sup_{\|\bm{\delta}\|_{p}\le \varepsilon} \mathcal{L}(\bm{\theta};\bm{x}+\bm{\delta},y)\,\Big], \label{eq:Rp-def}\\
R_c(\bm{\theta},\eta) &:= \mathbb{E}_{(\bm{x},y)\sim\mathcal{D}}\Big[\,\sup_{\substack{\bm{\delta}_c\in V_c(\bm{x})\\ \|\bm{\delta}_c\|_{p}\le \eta}} \mathcal{L}(\bm{\theta};\bm{x}+\bm{\delta}_c,y)\,\Big]. \label{eq:Rc-def}
\end{align}
Here, $\varepsilon,\eta>0$ are scalar radii, and $\bm{\delta},\bm{\delta}_c\in\mathbb{R}^d$ are perturbation vectors with $\bm{\delta}_c\in V_c(\bm{x})$.
We use the $\ell_p$ norm constraint for $R_c$ to enable direct comparison of local sensitivities. $R_c(\bm{\theta},\eta)$ serves as a first-order approximation of the robustness against the color transform family $\{ T_{\bm{\gamma}} \}$. Throughout, $\mathbb{E}[\cdot]$ denotes expectation over the joint draw $(\bm{x},y)\sim\mathcal{D}$; for brevity we omit the subscript.

\noindent\textbf{Geometric notions.}
We now introduce three geometric notions for an arbitrary subspace $V\subset\mathbb{R}^d$. These will form the analytical toolkit for our main results when specialized to the linearized color subspace, $V=V_c(\bm{x})$.

The restricted dual norm quantifies the first-order sensitivity of the loss to perturbations constrained to $V$:
\begin{equation*}
    \|\bm{g}\|_{q;V} \;:=\; \sup\{\langle \bm{g},\bm{z}\rangle:\ \bm{z}\in V,\ \|\bm{z}\|_{p}\le 1\}.
\end{equation*}
The orthogonal complement of the subspace is
\begin{equation*}
    V^\perp:=\{\bm{g}\in\mathbb{R}^d:\ \langle \bm{g},\bm{z}\rangle=0,\ \forall\,\bm{z}\in V\}.
\end{equation*}
If $\bm{g}(\bm{x})\in V^\perp$, then $\|\bm{g}(\bm{x})\|_{q;V}=0$, implying first-order invariance with respect to $V$. Finally, the $\ell_p$ distance from a vector $\bm{u}$ to $V$ is given by:
\begin{equation*}
    \operatorname{dist}_p(\bm{u},V):=\inf_{\bm{z}\in V}\|\bm{u}-\bm{z}\|_{p}.
\end{equation*}

\noindent\textbf{Assumptions.}
We use the following mild conditions.
\begin{enumerate}
    \item[(A1)] \emph{Regularity.} Standard smoothness and integrability conditions hold for $\mathcal{L}(\bm{\theta};\bm{x},y)$. The first-order Taylor remainder $r_{\bm{x}}(\bm{\delta})$ satisfies $\mathbb{E}\big[\sup_{\|\bm{\delta}\|_p \le r} |r_{\bm{x}}(\bm{\delta})|\big] = o(r)$ as $r\to 0$.
    \item[(A2)] \emph{Task margin.} We assume the existence of a task-relevant direction $\bm{u}(\bm{x})\neq \bm{0}$ and a required margin $C(\bm{x})>0$. With positive probability over $(\bm{x},y)\sim\mathcal{D}$, the input gradient $\bm{g}(\bm{x})$ must project onto this direction with at least this margin, satisfying: $$\langle \bm{g}(\bm{x}),\bm{u}(\bm{x})\rangle \;\ge\; C(\bm{x}).$$
    \item[(A3)] \emph{Misalignment.} With positive probability, $V_c(\bm{x}) \neq \{\bm{0}\}$ and $\bm{u}(\bm{x})$ is misaligned with $V_c(\bm{x})$:
    \begin{equation*}
    \operatorname{dist}_p\!\big(\bm{u}(\bm{x}),V_c(\bm{x})\big)\in\big(0,\ \|\bm{u}(\bm{x})\|_{p}\big).
    \end{equation*}
For $p=2$, this corresponds to $\bm{u}(\bm{x})$ being neither contained within $V_c(\bm{x})$ nor completely orthogonal to it; for general $p$, it imposes the analogous non-degenerate distance condition needed for the comparison below.
\end{enumerate}

\subsection{First-Order Incompatibility}

We analyze the first-order behavior of the robust risks.

\begin{lemma}[First-order risk expansions]
\label{lem:expansion}
Under (A1), robust risks admit first-order expansions at $\varepsilon=0$ and $\eta=0$:
\begin{equation}
\label{eq:exp-mean}
\begin{aligned}
R_p(\bm{\theta},\varepsilon)
&= \mathbb{E}[\mathcal{L}(\bm{\theta};\bm{x},y)] + \varepsilon\,\mathbb{E}\big[\|\bm{g}(\bm{x})\|_q\big] + o(\varepsilon),\\
R_c(\bm{\theta},\eta)
&= \mathbb{E}[\mathcal{L}(\bm{\theta};\bm{x},y)] + \eta\,\mathbb{E}\big[\|\bm{g}(\bm{x})\|_{q;V_c(\bm{x})}\big] + o(\eta).
\end{aligned}
\end{equation}
\end{lemma}

\begin{proof}
For any $(\bm{x}, y)$, the first-order Taylor expansion is
\begin{align*}
\mathcal{L}(\bm{\theta};\bm{x}+\bm{\delta},y) &= \mathcal{L}(\bm{\theta};\bm{x},y) + \langle \bm{g}(\bm{x}),\bm{\delta}\rangle + r_{\bm{x}}(\bm{\delta}).
\end{align*}
Maximizing the linearized robust loss for $R_p$ yields:
\begin{align*}
\sup_{\|\bm{\delta}\|_{p}\le\varepsilon} \big(\mathcal{L}(\bm{\theta};\bm{x},y) + \langle \bm{g}(\bm{x}),\bm{\delta}\rangle\big) &= \mathcal{L}(\bm{\theta};\bm{x},y) + \varepsilon\,\|\bm{g}(\bm{x})\|_q,
\end{align*}
by the definition of the dual norm. By (A1), the expectation of the remainder term is $o(\varepsilon)$. This confirms the expansion for $R_p(\bm{\theta},\varepsilon)$. The argument for $R_c(\bm{\theta},\eta)$ is analogous, utilizing the restricted dual norm.
\end{proof}

\begin{figure}[t]
\centering
\includegraphics[width=0.6\linewidth]{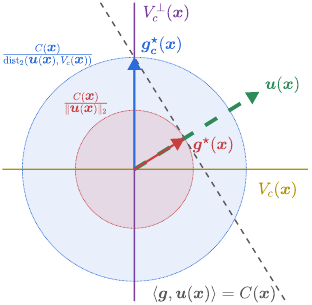}
\Description{A geometric illustration diagram showing the negative transfer effect and gradient norms in the Euclidean space.}
\caption{Geometric illustration of negative transfer at a fixed $\bm{x}$ in the Euclidean case ($p=q=2$). The horizontal axis is the linearized color subspace $V_c(\bm{x})$ and the vertical axis is its orthogonal complement $V_c^\perp(\bm{x})$. The dashed ray denotes the task-relevant direction $\bm{u}(\bm{x})$, and the dashed line is the margin constraint $\langle \bm{g},\bm{u}(\bm{x})\rangle = C(\bm{x})$, where $C(\bm{x})>0$ is the required margin. The arrow $\bm{g}^\star(\bm{x})$ is the unconstrained minimum-norm gradient that satisfies the margin, lying on the smallest feasible $\ell_2$ ball of radius $\frac{C(\bm{x})}{\|\bm{u}(\bm{x})\|_2}$ (inner shaded disk). Enforcing color invariance restricts $\bm{g}$ to $V_c^\perp(\bm{x})$, producing the invariant minimum-norm solution $\bm{g}^\star_c(\bm{x})$ on a larger ball of radius $\frac{C(\bm{x})}{\operatorname{dist}_2(\bm{u}(\bm{x}),V_c(\bm{x}))}$ (outer shaded disk). Under assumption (A3), this radius is strictly larger, illustrating the increased gradient norm required by color invariance.}
\label{fig:first-order-geometry}
\end{figure}

\noindent\textbf{Implication.}
Lemma~\ref{lem:expansion} highlights a fundamental tension. Minimizing $R_p$~\eqref{eq:Rp-def} requires minimizing the overall norm $\|\bm{g}(\bm{x})\|_q$. Minimizing $R_c$~\eqref{eq:Rc-def} requires minimizing the restricted norm $\|\bm{g}(\bm{x})\|_{q;V_c(\bm{x})}$, achieving color invariance, i.e., $\bm{g}(\bm{x})\in V_c^\perp(\bm{x})$.
We now show that enforcing this invariance increases the minimum required overall norm, formalized by analyzing the optimization under margin constraint (A2).

\begin{lemma}[Cost of invariance]
\label{lem:minnorm}
Fix $(\bm{x},y)$ as in (A2). Among all vectors $\bm{g}$ satisfying the margin constraint $\langle \bm{g},\bm{u}(\bm{x})\rangle\ge C(\bm{x})$, the following minimum norms hold:
\begin{align}
\text{Unconstrained:}\quad
\min_{\bm{g}\in\mathbb{R}^d}\ \|\bm{g}\|_q
&= \frac{C(\bm{x})}{\|\bm{u}(\bm{x})\|_{p}},
\label{eq:min-no-const}\\
\text{Color-invariant:}\quad
\min_{\bm{g}\in V_c^\perp(\bm{x})}\ \|\bm{g}\|_q
&\ge \frac{C(\bm{x})}{\operatorname{dist}_p\!\big(\bm{u}(\bm{x}),V_c(\bm{x})\big)} .
\label{eq:min-with-const}
\end{align}
The unconstrained problem \eqref{eq:min-no-const} admits a unique minimizer $\bm{g}^{\star}(\bm{x})$, and the color-invariant problem \eqref{eq:min-with-const} admits at least one minimizer, denoted $\bm{g}^{\star}_{c}(\bm{x})$.
\end{lemma}

\begin{proof}[Proof of Eq.~\eqref{eq:min-no-const}]
By H\"older's inequality, we have
\begin{equation*}
\langle \bm{g},\bm{u}(\bm{x})\rangle \le \|\bm{g}\|_q\,\|\bm{u}(\bm{x})\|_{p},
\end{equation*}
which implies $\|\bm{g}\|_q \ge \frac{C(\bm{x})}{\|\bm{u}(\bm{x})\|_{p}}$.
Equality holds when $\bm{g}$ satisfies the H\"older alignment condition with $\bm{u}(\bm{x})$.
\end{proof}

\begin{proof}[Proof of Eq.~\eqref{eq:min-with-const} (Lower bound)]
If $\bm{g}\in V_c^\perp(\bm{x})$, then $\langle \bm{g},\bm{z}\rangle=0$ for any $\bm{z}\in V_c(\bm{x})$. Hence $$\langle \bm{g},\bm{u}(\bm{x})\rangle
= \langle \bm{g},\bm{u}(\bm{x})-\bm{z}\rangle
\le \|\bm{g}\|_q\,\|\bm{u}(\bm{x})-\bm{z}\|_{p}.$$Taking $\inf_{\bm z\in V_c(\bm{x})}$ yields $\|\bm{g}\|_q \ge \frac{C(\bm{x})}{\operatorname{dist}_p(\bm{u}(\bm{x}),V_c(\bm{x}))}$.
\end{proof}

\begin{theorem}[First-order incompatibility]
\label{thm:incompatibility}
Under (A1)--(A3), with positive probability,
\[
\|\bm{g}^{\star}(\bm{x})\|_q \;<\; \|\bm{g}^{\star}_{c}(\bm{x})\|_q.
\]
Enforcing $\bm{g}(\bm{x})\in V_c^\perp(\bm{x})$ strictly increases the minimal $\|\bm{g}(\bm{x})\|_q$ required to maintain the margin. By Lemma~\ref{lem:expansion} and the expansion~\eqref{eq:exp-mean}, this raises the minimal achievable first-order robust-risk coefficient of $R_p$.
\end{theorem}

\begin{proof}
By Lemma~\ref{lem:minnorm}, from Eq.~\eqref{eq:min-no-const}, $\|\bm{g}^{\star}(\bm{x})\|_q = \frac{C(\bm{x})}{\|\bm{u}(\bm{x})\|_{p}}$.
From Eq.~\eqref{eq:min-with-const},
$\|\bm{g}^{\star}_{c}(\bm{x})\|_q \ge \frac{C(\bm{x})}{\operatorname{dist}_p(\bm{u}(\bm{x}),V_c(\bm{x}))}$.
By Assumption~(A3), we have $\operatorname{dist}_p(\bm{u}(\bm{x}),V_c(\bm{x})) < \allowbreak \|\bm{u}(\bm{x})\|_{p}$, hence the claim.
\end{proof}

\noindent\textbf{Bidirectional conflict and negative transfer.}
Theorem~\ref{thm:incompatibility} elucidates why a joint objective leads to suboptimal trade-offs.
The geometric intuition is illustrated in Fig.~\ref{fig:first-order-geometry} for the Euclidean case ($p=2$).
When the task direction $\bm{u}(\bm{x})$ is misaligned with the color subspace $V_c(\bm{x})$, enforcing invariance imposes a cost.
\noindent\emph{Color robustness harms $\ell_p$ robustness:} Achieving color robustness requires driving $\bm{g}(\bm{x})$ toward the orthogonal complement $V_c^\perp(\bm{x})$.
This constraint forces the gradient away from the unconstrained optimum $\bm{g}^{\star}(\bm{x})$, increasing the minimal required gradient norm (as $\|\bm{g}^{\star}_c(\bm{x})\|_q > \|\bm{g}^{\star}(\bm{x})\|_q$).
This raises the lower bound on the first-order term of $R_p$.
\noindent\emph{$\ell_p$ robustness harms color robustness:} Conversely, minimizing $R_p$ pushes $\bm{g}(\bm{x})$ toward $\bm{g}^{\star}(\bm{x})$.
By Theorem~\ref{thm:incompatibility}, $\bm{g}^{\star}(\bm{x})$ does not lie in $V_c^\perp(\bm{x})$.
Therefore, the restricted norm $\|\bm{g}^{\star}(\bm{x})\|_{q;V_c(\bm{x})}$ is strictly positive, meaning the first-order color risk does not vanish.
Consequently, JAT can trade off the two objectives, producing suboptimal solutions compared to specialized single-objective training when the incompatibility is pronounced.

\begin{figure*}[t]
\centering
\includegraphics[width=1.0\linewidth]{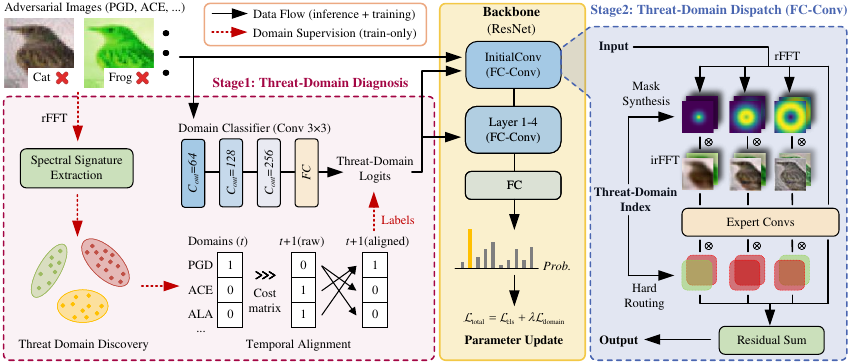}
\Description{A block diagram showing the two-stage workflow of the TaFD defense framework: Diagnosis and Dispatch.}
\caption{\textbf{Overview of the TaFD Defense Framework.} TaFD employs a two-stage Diagnosis--Dispatch architecture. \textbf{Stage 1 (Diagnosis):} Spectral signatures are extracted via rFFT and clustered to discover latent threat domains; the resulting assignments are stabilized via Temporal Alignment and supervise a lightweight Domain Classifier that predicts the Threat-Domain Index. \textbf{Stage 2 (Dispatch):} Conditioned on this index, FC-Conv modules synthesize threat-domain-specific spectral masks in the frequency domain (rFFT/irFFT), dispatch masked features to the corresponding expert, and fuse outputs with a residual pathway, thereby enforcing frequency-grounded structural parameter separation and mitigating optimization conflicts.}
\label{fig:workflow}
\end{figure*}

\section{Methodology} \label{sec:methodology}
Our theoretical analysis establishes a fundamental incompatibility: heterogeneous threats impose conflicting demands on the input-gradient field, leading to robustness trade-offs. This intrinsic conflict necessitates a shift from unified optimization to explicit decoupling of these competing objectives. Guided by the observation that conflicting threats leave separable signatures in the frequency domain, such as high-frequency texture perturbations versus low-frequency color shifts, we implement this decoupling strategy by recasting joint defense as a Diagnosis--Dispatch framework. The diagnosis stage identifies the latent threat domain of each sample, generating a compact conditioning signal. The dispatch stage leverages this signal to conditionally process spectral components via a novel Frequency-Conditional Convolution (FC-Conv) module, thereby enforcing structured parameter separation and mitigating gradient conflicts. Fig.~\ref{fig:workflow} illustrates the complete workflow of our TaFD defense framework.

\subsection{Threat-Domain Diagnosis via Data-Driven Attack-Type Partitioning}
\label{sec:domain_diagnosis}
The threat-domain diagnosis stage aims to generate accurate sample-level conditioning signals for subsequent dispatch. However, reliance on explicit attack labels is insufficient.
While treating each attack type as a distinct domain may appear straightforward, the expanding attack landscape would lead to unbounded model complexity. Moreover, distinct attacks often exhibit substantial spectral overlap, rendering dedicated per-attack resources redundant.
Robust conditioning signals must therefore emerge from a more principled partitioning based on spectral characteristics. To this end, we develop a data-driven approach: first, we define a discriminative spectral feature space where attacks exhibit separable signatures (Section~\ref{sec:spectral_signatures}); second, we learn stable attack-type prototypes and perform unsupervised clustering to discover latent threat domains (Section~\ref{sec:prototype_learning}). This automated partitioning, coupled with temporal alignment (Section~\ref{sec:temporal_alignment}), provides high-confidence supervision for a lightweight classifier (Section~\ref{sec:lightweight_classifier}) that generates conditioning signals for guiding specialized processing flows (Section~\ref{sec:frequency_dispatch}).

\subsubsection{Spectral Signatures of Adversarial Attacks} \label{sec:spectral_signatures}
Heterogeneous threats exhibit distinct spectral signatures in the frequency domain. We capture these via a feature extraction function $\bm{\phi}(\cdot)$ that maps input $\bm{x}$ to a compact feature vector $\bm{\phi}(\bm{x}) \in \mathbb{R}^{N_\phi}$ characterizing its spectral energy distribution, where $N_\phi$ denotes the feature dimension. These features establish a stable, discriminative space for unsupervised threat-domain partitioning.

Given an input image $\bm{x} \in \mathbb{R}^{H \times W \times 3}$, we apply the two-dimensional rFFT $\mathcal{F}_R(\bm{x})$, then convert the RGB spectrum to YCbCr color space in the frequency domain, decoupling luminance (Y) from chrominance (Cb, Cr). We extract $\bm{\phi}(\bm{x})$ comprising handcrafted spectral statistics from the YCbCr energy distribution: (i) normalized energy proportions within low-, mid-, and high-frequency bands for luminance and chrominance channels; (ii) chroma angles characterizing hue per band; and (iii) chroma dominance metrics quantifying relative energy between Cb and \allowbreak Cr channels.

\subsubsection{Attack-Type Prototype Learning and Threat-Domain Discovery} \label{sec:prototype_learning}
To capture the spectral signature of each known attack type, we maintain a set of attack-type prototypes $\{\bm{p}_i\}_{i=1}^N$, where $N$ denotes the number of distinct attack types in the training set $\mathcal{T} = \{T_1, \dots, T_N\}$. Each prototype $\bm{p}_i \in \mathbb{R}^{N_\phi}$ represents the centroid of attack type $T_i$'s distribution in the spectral feature space, updated via exponential moving average.

At training iteration \(t\), we define $\bar{\bm{\phi}}_i^{(t)}$ as the average feature vector for samples corresponding to attack type $T_i$ in the current batch. The prototype for $T_i$ is updated as:
\begin{equation*}
\bm{p}_i^{(t)} = \beta \,\bm{p}_i^{(t-1)} + (1-\beta)\, \bar{\bm{\phi}}_i^{(t)},
\end{equation*}
where $\beta \in [0,1)$ is the momentum coefficient. Periodically (every $U$ iterations), we cluster the stabilized prototypes to discover latent threat domains. We standardize prototypes, project them to a lower-dimensional subspace ($d=4$) via PCA, then apply K-means clustering with $K$ target domains. This yields a mapping $\mathcal{M}^{(t)}:\mathcal{T}\rightarrow\{1,\dots,K\}$ assigning threat-domain indices to each attack type $T_i\in\mathcal{T}$.

\subsubsection{Temporal Alignment via Hungarian Matching} \label{sec:temporal_alignment}
Raw mappings $\mathcal{M}^{(t)}$ obtained from iterative unsupervised clustering exhibit label switching, whereby cluster indices permute arbitrarily across successive updates. To preserve stable supervisory signals throughout training, we align each raw mapping $\mathcal{M}^{(t)}$ with the preceding stabilized mapping $\hat{\mathcal{M}}^{(t-1)}$ by applying the Hungarian algorithm to identify the optimal permutation that minimizes cross-iteration assignment discrepancies.

We construct a mismatch cost matrix $\bm{C}^{\text{cost}} \in \mathbb{R}^{K \times K}$, where each element $C_{ij}^{\text{cost}}$ quantifies the disagreement between the current assignment to cluster $i$ and the previous stable assignment to cluster $j$:
\begin{equation*}
    C_{ij}^{\text{cost}} = \left|\left\{ T_k \in \mathcal{T} \mid \mathcal{M}^{(t)}(T_k)=i, \ \hat{\mathcal{M}}^{(t-1)}(T_k)\neq j \right\}\right|,
\end{equation*}
where $|\cdot|$ denotes set cardinality. Thus, $C_{ij}^{\text{cost}}$ represents the number of attack types currently assigned to cluster $i$ that were not previously assigned to stable domain $j$.

The Hungarian algorithm determines the optimal permutation matrix $\bm{P}^*$ by minimizing the total alignment cost over the set of all $K \times K$ permutation matrices $\mathcal{P}_K$:
\begin{equation*}
    \bm{P}^* = \operatorname*{arg\,min}_{\bm{P} \in \mathcal{P}_K} \sum_{i=1}^K \sum_{j=1}^K C_{ij}^{\text{cost}} P_{ij}.
\end{equation*}
Applying this optimal permutation to the raw cluster indices yields the temporally aligned mapping $\hat{\mathcal{M}}^{(t)}$. This procedure ensures that each threat-domain index maintains correspondence with a consistent subset of attack types throughout training, thereby stabilizing the optimization landscape.

\subsubsection{Lightweight Domain Classifier} \label{sec:lightweight_classifier}
The stabilized mapping $\hat{\mathcal{M}}^{(t)}$ provides reliable pseudo-labels that reflect the latent threat-domain structure. We leverage this mapping as supervision for training a lightweight domain classifier $G$ parameterized by $\bm{\theta}_G$. Implemented as a shallow convolutional neural network, the classifier adopts a streamlined architecture comprising three convolutional blocks (each consisting of convolution, batch normalization, ReLU activation, dropout, and max pooling) followed by global average pooling and a two-layer fully connected classification head. The classifier is trained by minimizing the cross-entropy loss over the training distribution $\mathcal{S}$, which comprises adversarial examples $\bm{x}$ and their corresponding generating attack types $T$. The stabilized mapping $\hat{\mathcal{M}}^{(t)}(T)$ serves as the pseudo-label:
\begin{equation}
  \mathcal{L}_{\text{domain}} = \mathbb{E}_{(\bm{x}, T) \sim \mathcal{S}} \big[ \mathcal{H}\big(G(\bm{x}; \bm{\theta}_G), \hat{\mathcal{M}}^{(t)}(T)\big) \big],
  \label{eq-diagnoser-loss}
\end{equation}
where $\mathcal{H}$ denotes the cross-entropy loss. During inference, the classifier predicts the threat-domain index $k^*_{\text{pred}} \in \{1, \dots, K\}$ for input $\bm{x}$ via:
\begin{equation*}
 k^*_{\text{pred}} = \operatorname*{arg\,max}_k \ [G(\bm{x}; \bm{\theta}_G)]_k,
\end{equation*}
where $[\cdot]_k$ denotes the $k$-th element of the output vector.

\subsection{Threat-Domain-Dispatch via Frequency-Conditional Convolution} \label{sec:frequency_dispatch}
The dispatch stage utilizes the diagnosed threat-domain index to activate specialized processing pathways tailored to the discovered threat domain. To achieve this specialization, we introduce the FC-Conv module, which strategically replaces conventional spatial convolutions at key hierarchical stages. The FC-Conv module processes an input feature map $\bm{X}$ to produce an output $\bm{Y}$. It performs a conditional, two-stage process guided by the diagnosed threat-domain index, which we denote generically as $k^*$. The first stage modulates the network's spectral response via threat-domain-conditioned spectral masking (Section~\ref{sec:Domain_Conditioned_Spectral_Masking}), operating in the frequency domain via the channel-wise real-valued Fast Fourier Transform $\mathcal{F}_R$. In the second stage, following transformation back to the spatial domain using the inverse rFFT $\mathcal{F}_R^{-1}$, the module dispatches the features to a dedicated threat-domain expert via sample-level hard routing (Section~\ref{sec:Dedicated_Expert_Processing}).

\subsubsection{Threat-Domain-Conditioned Spectral Masking via Compact Basis Expansion} \label{sec:Domain_Conditioned_Spectral_Masking}
To achieve threat-domain-specific specialization, the frequency spectrum must be adaptively decomposed per the conditioning signal $k^*$, requiring $K$ distinct spectral masks. However, directly learning these masks in the high-dimensional frequency domain requires numerous parameters, causing overfitting and non-smooth masks that induce ringing.

To address these challenges, we adopt a compact, basis-driven parameterization, synthesizing masks via a low-dimensional coefficient space conditioned on $k^*$. This approach inherently enforces smoothness constraints and significantly reduces the parameter count. We employ pre-defined Zernike polynomials $\{B_l(\boldsymbol{\omega})\}_{l=1}^{N_{\text{basis}}}$ as basis functions, where $\boldsymbol{\omega}$ denotes 2D frequency coordinates and $l$ indexes the basis. This basis is constructed from a radial degree ($n$) and an azimuthal degree ($m$), where $N_{\text{basis}}$ is the total number of $(n,m)$ combinations (a hyperparameter defining mask complexity). Zernike polynomials form an orthogonal basis over the unit disk, making them well-suited for capturing these radial and angular patterns. We normalize the rFFT output frequency coordinates so that the maximum radius maps to 1, conforming to the unit disk domain.

Given the conditioning signal $k^*$, the spectral logits $S_k(\boldsymbol{\omega}; k^*)$ for expert $k \in \{1, \dots, K\}$ are synthesized as a linear combination of these basis functions:
\begin{equation}
S_k(\boldsymbol{\omega}; k^*) = \sum_{l=1}^{N_{\text{basis}}} \alpha_{k,l}(k^*) B_l(\boldsymbol{\omega}) + b_k(k^*).
\label{eq-basis-expansion}
\end{equation}
The coefficients $\alpha_{k,l}(k^*)$ and bias $b_k(k^*)$ are learnable parameters retrieved from an embedding layer indexed by $k^*$. We define the spectral mask tensors $\{\bm{M}_k(k^*)\}_{k=1}^K$, whose elements $M_k(\boldsymbol{\omega};k^*)$ are obtained via Softmax over the $K$ expert logits at each frequency $\boldsymbol{\omega}$. This ensures the masks form a partition of unity, $\sum_{k=1}^{K} M_k(\boldsymbol{\omega};k^*) = 1$ for all $\boldsymbol{\omega}$:
\begin{equation}
M_k(\boldsymbol{\omega}; k^*) = \frac{\exp(S_k(\boldsymbol{\omega}; k^*))}{\sum_{j=1}^{K}\exp(S_j(\boldsymbol{\omega}; k^*))}.
\label{eq-spectral-softmax}
\end{equation}

\subsubsection{Dedicated Expert Processing and Sample-Level Hard Routing} \label{sec:Dedicated_Expert_Processing}
The generated mask tensors decompose the spectrum $\hat{\bm{X}}=\mathcal{F}_R(\bm{X})$ into $K$ specialized components via element-wise product:
\begin{equation}
    \hat{\bm{X}}_k = \hat{\bm{X}} \odot \bm{M}_k(k^*),
    \label{eq-decomposition}
\end{equation}
where $\hat{\bm{X}}$ is complex-valued from the rFFT, while $\bm{M}_k(k^*)$ are real-valued non-negative masks that scale spectral magnitudes while preserving phase and broadcast across channels. These components are then transformed back via inverse rFFT, producing $\bm{X}_k=\mathcal{F}_R^{-1}(\hat{\bm{X}}_k)$. We maintain $K$ dedicated experts $\{\mathcal{E}_k\}_{k=1}^{K}$ with one-to-one correspondence to threat domains. Each resulting spatial component $\bm{X}_k$ is subsequently processed by its corresponding expert: $\bm{Y}_k=\mathcal{E}_k(\bm{X}_k)$.

Crucially, while spectral masking decomposes features, mitigating negative transfer requires decoupling optimization. We enforce structural parameter separation via sample-level hard routing: given conditioning signal $k^*$, the module selects the output of the corresponding expert $\mathcal{E}_{k^*}$. Formally, the routed output $\bm{Y}_{\text{routed}}$ is expressed using the indicator function $\mathbb{I}(\cdot) \in \{0, 1\}$:
\begin{equation}
    \bm{Y}_{\text{routed}} = \sum_{k=1}^{K} \mathbb{I}(k=k^*) \cdot \bm{Y}_k.
    \label{eq-hard-routing-indicator}
\end{equation}
Since the indicator function is non-differentiable, we employ BPDA~\cite{DBLP:conf/icml/AthalyeC018} with a softmax-weighted proxy during adversarial evaluation to ensure accessible gradients.

To ensure optimization stability with this hard routing mechanism, the source of the conditioning signal $k^*$ differs between training and inference. During training, we utilize the stabilized pseudo-label derived from the known attack type $T$: $k^* = \hat{\mathcal{M}}^{(t)}(T)$. During inference, we use the classifier's prediction $k^* = k^*_{\text{pred}}$. This hard routing mechanism is core to our decoupling strategy. It ensures that for any given input, only the parameters of the designated expert are activated and updated, directly mitigating optimization conflicts and promoting threat-domain specialization. The final output of the FC-Conv combines the routed expert response with a parallel residual pathway:
\begin{equation}
    \bm{Y} = \operatorname{Conv}_{\text{res}}(\bm{X}) + \bm{Y}_{\text{routed}},
    \label{eq-final-output}
\end{equation}
where $\operatorname{Conv}_{\text{res}}$ is a standard convolution matching dimensions with $\bm{Y}_{\text{routed}}$.

\subsection{Architectural Integration and Objective}\label{sec:architecture_optimization}
To decouple heterogeneous threats while maintaining efficiency, we adopt a hierarchical injection strategy. We replace the standard convolution at the entry point of each feature-extraction stage in the network hierarchy with an FC-Conv module. This strategy applies to both the initial network stem and the first convolutional layer of each subsequent feature block, while all other layers retain standard convolutions. This placement ensures that frequency-domain decoupling precedes feature processing within each stage, thereby reducing feature entanglement across abstraction levels.

The framework is trained end-to-end via joint optimization of the integrated classification backbone and the lightweight domain classifier. The total loss combines the primary classification loss $\mathcal{L}_{\text{cls}}$ and the auxiliary threat-domain diagnosis loss $\mathcal{L}_{\text{domain}}$ (Eq.~\eqref{eq-diagnoser-loss}):
\begin{equation}
    \mathcal{L}_{\text{total}} = \mathcal{L}_{\text{cls}} + \lambda \mathcal{L}_{\text{domain}},
    \label{eq-total-loss}
\end{equation}
where $\lambda>0$ balances the two objectives. The primary classification loss, $\mathcal{L}_{\text{cls}}$, is defined as the standard cross-entropy loss computed on the generated adversarial examples against their true labels. This joint objective simultaneously supervises the domain classifier to provide accurate threat-domain signals while driving the backbone to leverage these signals for robust classification.

\section{Experiments}
\label{sec:experiments}

This section presents a comprehensive evaluation of the proposed TaFD, designed to address key challenges in heterogeneous joint adversarial training (JAT). Our experimental design focuses on the following four research questions (RQs):

\begin{itemize}
    \item \textbf{RQ1 (Robustness against Canonical Heterogeneous Threats):} How effectively does TaFD achieve balanced robustness against the canonical union of $\ell_p$-bounded and color threats compared to state-of-the-art methods?
    \item \textbf{RQ2 (Generalizability to Broader Threats):} Does the robustness of TaFD generalize to broader heterogeneous threat unions encompassing spatial and parametric threats?
    \item \textbf{RQ3 (Robustness of the Diagnosis Module):} Does attacking the diagnosis module compromise TaFD's robustness?
    \item \textbf{RQ4 (Ablation Studies and Hyperparameter Sensitivity Analysis):} What are the contributions of TaFD's core components, and how does the number of latent threat domains ($K$) impact robustness?
\end{itemize}

\subsection{Experimental Setup}
\subsubsection{Datasets and Network Architectures}
We evaluate TaFD on CIFAR-10/100~\cite{krizhevsky2009learning} and Tiny-ImageNet~\cite{le2015tiny} (resized to $32\times 32$ following~\cite{DBLP:journals/pami/ZhangYSY24}). To assess generalization across architectures, we integrate TaFD into ResNet-34~\cite{DBLP:conf/cvpr/HeZRS16}, a convolutional architecture, and MobileViT-XS~\cite{DBLP:conf/iclr/MehtaR22}, a hybrid transformer. Within the TaFD framework, FC-Conv modules replace the primary $3\times3$ spatial convolution at two key positions: the network stem and the first block of each feature extraction stage. Specifically, for ResNet-34, replacements occur at the stem and the first residual block in layers 1--4; for MobileViT-XS, replacements occur at the stem and the first MV2Block of each stage.

\subsubsection{Threat Models}
For RQ1 (Canonical Heterogeneous Threats), we consider the union of $\ell_p$-bounded attacks ($\ell_\infty$-APGD and $\ell_2$-APGD~\cite{DBLP:conf/icml/Croce020a}) and diverse semantic color transformations: ACE~\cite{DBLP:journals/tifs/00010L23}, ALA~\cite{DBLP:conf/mm/0001S0JZFLP23}, HSVAdv~\cite{DBLP:conf/cvpr/HosseiniP18}, ReColorAdv~\cite{DBLP:conf/nips/LaidlawF19}, and RetouchUAA~\cite{DBLP:journals/tcsv/XieHQF25}. For RQ2 (Diverse Heterogeneous Unions), we further evaluate TaFD on unions involving spatial flow-based transformations (StAdv~\cite{DBLP:conf/iclr/XiaoZ0HLS18}) and parametric on-manifold attacks (GPGD~\cite{DBLP:journals/pr/XiaoYFWL25}).

\subsubsection{Baselines}
We compare TaFD with state-of-the-art defense methods, including JAT approaches such as Mean-AT~\cite{tramer2019adversarial}, MNG~\cite{DBLP:conf/icml/MadaanSH21}, GBN~\cite{liu2024towards}, RAMP~\cite{jiang2024ramp}, and PUAT~\cite{DBLP:journals/pami/ZhangYSY24}, as well as recent frequency-domain defenses FACE~\cite{niu2023defense} and DAT~\cite{DBLP:conf/nips/LiLW0024}. We also extend the widely used TRADES~\cite{DBLP:conf/icml/ZhangYJXGJ19} for the JAT setting, since its objective explicitly balances natural accuracy and adversarial robustness.

\subsubsection{Implementation Details}
We train models with Adam for $75$ epochs with batch size $128$ and initial learning rate $0.001$, decayed by $10\times$ at epochs $50$ and $70$. For JAT, we generate one adversarial example per attack type per clean input and optimize objective Eq.~\eqref{eq-total-loss} with $\lambda=1.0$.

\noindent\textbf{Baseline Training Strategies.}
Mean-AT, TRADES, MNG, GBN, RAMP, and DAT are trained with the same JAT strategy as TaFD. For PUAT and FACE, we follow their original training protocols, as their architectures are incompatible with standard multi-threat JAT: PUAT relies on GAN-based distribution alignment, while FACE employs frequency-adaptive input reconstruction.

\noindent\textbf{Attack Parameters.}
During training, we use 10-step attacks for efficiency. For rigorous evaluation, we employ 100-step AutoPGD for $\ell_\infty$ and $\ell_2$ attacks; all remaining attacks retain their training-time configurations. For $\ell_\infty$ PGD, $\varepsilon=8/255$ with step size $\alpha=2/255$; for $\ell_2$ PGD, $\varepsilon=0.5$ with step size $\alpha=0.1$. On CIFAR, color attack step sizes are $0.01$ for ReColorAdv, $0.1$ for RetouchUAA, and $1.0$ for all others. On Tiny-ImageNet, these step sizes are scaled by $0.1$ to maintain perceptual plausibility given the lack of explicit perturbation bounds. For all other attack hyperparameters, we adhere to the default settings provided in their official implementations. To preclude gradient masking, all attacks are generated with BPDA enabled for TaFD's hard routing.

\noindent\textbf{TaFD Hyperparameters.}
We set $K=6$ threat domains. For the diagnosis stage, the prototype momentum is $\beta=0.99$, and the domain mapping is updated every $U=50$ iterations. Spectral feature extraction and prototype updates are performed only during the first 10 epochs to stabilize the spectral prototypes. FC-Conv modules use Zernike bases up to radial and azimuthal degree $3$ ($n_{\max}=m_{\max}=3$; $N_{\text{basis}}=10$).

\subsubsection{Evaluation Metrics}
We report Clean Accuracy (CA) and Robust Accuracy (RA) for each evaluated attack. Average Robust Accuracy (Avg-RA) over the threat union is our primary metric of balanced robustness.

\begin{table*}[t!]
\centering
\footnotesize
\setlength{\tabcolsep}{3.5pt}
\renewcommand{\arraystretch}{1.0}
\caption{\textbf{Robustness against Canonical Threats on ResNet-34 and MobileViT-XS.}
Clean and robust accuracy (\%) on CIFAR-10/100 and Tiny-ImageNet under $\ell_\infty$-APGD, $\ell_2$-APGD, and color-based attacks (ACE, ALA, HSVAdv, ReColorAdv, RetouchUAA). \textbf{Avg-RA} averages the seven attacks excluding Clean. Cells in \begingroup\setlength{\fboxsep}{1.2pt}\legendbox{bestcol}{gold}\endgroup/\begingroup\setlength{\fboxsep}{1.2pt}\legendbox{secondcol}{light gold}\endgroup\ mark the column-wise \begingroup\setlength{\fboxsep}{1.2pt}\legendbox{bestcol}{best}\endgroup/\begingroup\setlength{\fboxsep}{1.2pt}\legendbox{secondcol}{second-best}\endgroup.}
\label{tab:results_combined_archetypal}
\begin{tabular*}{\linewidth}{@{\extracolsep{\fill}}lllccccccccc}
\toprule
\multirow{2}{*}{\textbf{Dataset}} & \multirow{2}{*}{\textbf{Model}} & \multirow{2}{*}{\textbf{Method}} & \multirow{2}{*}{\textbf{Clean}} & \multicolumn{2}{c}{$\ell_p$} & \multicolumn{5}{c}{\textbf{Color Attacks}} & \multirow{2}{*}{\textbf{Avg-RA}} \\
\cmidrule(lr){5-6}\cmidrule(lr){7-11}
& & & & $\ell_\infty$-APGD & $\ell_2$-APGD & ACE & ALA & HSVAdv & ReColorAdv & RetouchUAA & \\
\midrule
\multirow{18}{*}{\textbf{\textit{CIFAR-10}}} & \multirow{9}{*}{ResNet-34} & Mean-AT~\cite{tramer2019adversarial} & \second{93.7} & \second{26.0} & 47.5 & \second{61.6} & \second{74.7} & \second{59.2} & \second{49.7} & \second{64.1} & \second{54.7} \\
 & & TRADES~\cite{DBLP:conf/icml/ZhangYJXGJ19} & \best{94.2} & 21.8 & 42.1 & 58.6 & 73.6 & 55.8 & \best{51.5} & 62.1 & 52.2 \\
 & & MNG~\cite{DBLP:conf/icml/MadaanSH21} & 92.7 & 1.7 & 17.8 & 60.2 & 71.5 & 54.6 & 6.7 & 60.6 & 39.0 \\
 & & RAMP~\cite{jiang2024ramp} & 85.0 & 3.3 & 25.6 & 38.4 & 58.0 & 45.7 & 2.8 & 40.2 & 30.6 \\
 & & GBN~\cite{liu2024towards} & 90.5 & 0.1 & 1.2 & 50.2 & 54.8 & 40.2 & 22.3 & 51.5 & 31.5 \\
 & & PUAT~\cite{DBLP:journals/pami/ZhangYSY24} & 79.3 & 25.5 & \best{54.3} & 1.2 & 17.1 & 3.5 & 7.1 & 1.2 & 15.7 \\
 & & FACE~\cite{niu2023defense} & 93.4 & 0.0 & 0.0 & 2.0 & 5.8 & 28.4 & 0.0 & 3.3 & 5.6 \\
 & & DAT~\cite{DBLP:conf/nips/LiLW0024} & 61.1 & 0.2 & 7.0 & 20.5 & 32.0 & 19.2 & 0.7 & 20.5 & 14.3 \\
 & & \textbf{TaFD (Ours)} & \second{93.7} & \best{34.4} & \second{51.0} & \best{69.6} & \best{80.2} & \best{68.7} & 47.2 & \best{70.3} & \best{60.2} \\
\cmidrule(lr){2-12}
& \multirow{9}{*}{MobileViT-XS} & Mean-AT & 84.5 & 38.9 & \second{66.5} & 35.9 & 47.8 & 30.5 & 25.6 & 30.3 & 39.4 \\
& & TRADES & 89.7 & 29.7 & 64.7 & 38.3 & \second{51.9} & \second{32.4} & 16.7 & \second{34.1} & 38.3 \\
& & MNG & 83.2 & \second{40.1} & \best{67.1} & \second{40.0} & 49.8 & 32.3 & \second{30.8} & 33.6 & \second{42.0} \\
& & RAMP & 70.5 & 35.8 & 53.9 & 4.6 & 19.9 & 7.9 & 24.3 & 4.6 & 21.6 \\
& & GBN & 65.5 & 0.0 & 0.0 & 1.5 & 10.8 & 2.8 & 0.1 & 1.7 & 2.4 \\
& & PUAT & 80.2 & 24.8 & 53.8 & 1.5 & 16.8 & 3.8 & 6.0 & 1.3 & 15.4 \\
& & FACE & \second{91.7} & 0.0 & 1.3 & 1.0 & 28.5 & 4.8 & 0.1 & 1.5 & 5.3 \\
& & DAT & 69.8 & 33.5 & 54.2 & 27.0 & 34.3 & 21.8 & 23.5 & 22.2 & 30.9 \\
& & \textbf{TaFD (Ours)} & \best{92.7} & \best{40.8} & 63.4 & \best{52.6} & \best{69.8} & \best{47.5} & \best{47.0} & \best{51.9} & \best{53.3} \\
\midrule
\multirow{18}{*}{\textbf{\textit{CIFAR-100}}} & \multirow{9}{*}{ResNet-34} & Mean-AT~\cite{tramer2019adversarial} & 73.5 & 7.3 & 24.3 & \second{30.3} & 44.9 & \second{26.2} & 16.7 & 29.8 & \second{25.6} \\
 & & TRADES~\cite{DBLP:conf/icml/ZhangYJXGJ19} & \best{77.2} & 6.6 & 18.0 & 29.3 & \second{47.5} & 25.9 & \best{21.5} & \second{30.7} & \second{25.6} \\
 & & MNG~\cite{DBLP:conf/icml/MadaanSH21} & 71.7 & 4.7 & 16.4 & \second{30.3} & 43.6 & 25.3 & 8.5 & 29.4 & 22.6 \\
 & & RAMP~\cite{jiang2024ramp} & 51.3 & \best{25.5} & \best{36.7} & 2.9 & 8.7 & 3.3 & 14.4 & 1.8 & 13.3 \\
 & & GBN~\cite{liu2024towards} & 65.0 & 0.0 & 0.2 & 15.7 & 21.2 & 11.2 & 2.7 & 16.8 & 9.7 \\
 & & PUAT~\cite{DBLP:journals/pami/ZhangYSY24} & 50.2 & 8.5 & 26.0 & 0.3 & 5.3 & 1.2 & 1.3 & 0.2 & 6.1 \\
 & & FACE~\cite{niu2023defense} & 74.5 & 0.0 & 0.0 & 0.3 & 9.5 & 0.8 & 0.0 & 0.4 & 1.6 \\
 & & DAT~\cite{DBLP:conf/nips/LiLW0024} & 51.0 & 1.8 & 9.8 & 12.2 & 22.4 & 12.4 & 2.7 & 11.9 & 10.5 \\
 & & \textbf{TaFD (Ours)} & \second{76.9} & \second{13.6} & \second{33.8} & \best{37.4} & \best{51.2} & \best{33.4} & \second{19.8} & \best{34.2} & \best{31.9} \\
\cmidrule(lr){2-12}
& \multirow{9}{*}{MobileViT-XS} & Mean-AT & 60.7 & 17.8 & \second{40.1} & 13.0 & 21.8 & 9.8 & 8.8 & 10.0 & 17.3 \\
& & TRADES & 66.8 & 11.2 & 37.3 & 12.5 & \second{23.2} & 9.7 & 4.7 & 9.9 & 15.5 \\
& & MNG & 57.9 & 19.1 & 39.8 & \second{13.9} & 21.5 & \second{10.2} & 10.0 & \second{10.6} & \second{17.9} \\
& & RAMP & 40.1 & \second{21.3} & 30.1 & 3.2 & 7.3 & 3.9 & \second{11.7} & 2.3 & 11.4 \\
& & GBN & 22.9 & 0.0 & 0.0 & 0.0 & 1.8 & 0.6 & 0.0 & 0.2 & 0.4 \\
& & PUAT & 44.9 & 8.9 & 19.9 & 0.3 & 5.8 & 1.2 & 1.7 & 0.2 & 5.4 \\
 & & FACE & \second{67.9} & 0.0 & 0.1 & 0.2 & 8.9 & 0.8 & 0.0 & 0.2 & 1.5 \\
& & DAT & 50.9 & 17.2 & 34.9 & 9.1 & 16.6 & 7.2 & 8.2 & 6.6 & 14.3 \\
& & \textbf{TaFD (Ours)} & \best{71.4} & \best{22.4} & \best{41.1} & \best{21.6} & \best{39.2} & \best{19.3} & \best{21.8} & \best{21.0} & \best{26.6} \\
\midrule
\multirow{18}{*}{\textbf{\textit{Tiny-ImageNet}}} & \multirow{9}{*}{ResNet-34} & Mean-AT~\cite{tramer2019adversarial} & 45.1 & 3.6 & \best{20.0} & 12.1 & 25.2 & 4.6 & \second{18.6} & 10.1 & \second{13.5} \\
 & & TRADES~\cite{DBLP:conf/icml/ZhangYJXGJ19} & \best{51.0} & 1.0 & 10.4 & 11.8 & \best{28.1} & 4.7 & 12.9 & 10.7 & 11.4 \\
 & & MNG~\cite{DBLP:conf/icml/MadaanSH21} & 43.5 & 4.4 & \second{19.9} & \best{13.7} & 25.3 & 5.3 & \best{19.2} & \second{11.0} & \best{14.1} \\
 & & RAMP~\cite{jiang2024ramp} & 30.5 & \best{11.7} & 17.5 & 1.6 & 12.8 & 1.0 & 18.5 & 2.0 & 9.3 \\
 & & GBN~\cite{liu2024towards} & 36.3 & 0.0 & 0.0 & 3.9 & 15.4 & 1.2 & 0.9 & 3.7 & 3.6 \\
 & & PUAT~\cite{DBLP:journals/pami/ZhangYSY24} & 2.0 & 0.4 & 1.1 & 0.3 & 1.1 & 0.5 & 1.0 & 0.4 & 0.7 \\
 & & FACE~\cite{niu2023defense} & 47.3 & 0.0 & 0.0 & 0.4 & 0.3 & \best{18.0} & 0.4 & 0.4 & 2.8 \\
 & & DAT~\cite{DBLP:conf/nips/LiLW0024} & 39.3 & 0.2 & 5.4 & \second{13.0} & 24.6 & \second{6.5} & 11.8 & \best{12.3} & 10.5 \\
 & & \textbf{TaFD (Ours)} & \second{49.3} & \second{6.3} & 19.8 & 5.8 & \second{26.9} & 3.2 & 17.1 & 6.2 & 12.2 \\
\cmidrule(lr){2-12}
& \multirow{9}{*}{MobileViT-XS} & Mean-AT & 40.4 & 4.4 & \second{20.7} & 6.8 & 21.6 & 2.6 & \second{17.3} & 5.7 & 11.3 \\
& & TRADES & \second{45.8} & 1.0 & 15.4 & 7.5 & \second{23.9} & 2.7 & 14.7 & \second{6.7} & 10.3 \\
& & MNG & 38.4 & 4.2 & 20.4 & \second{7.7} & 21.3 & \second{3.6} & 16.8 & 6.4 & \second{11.5} \\
& & RAMP & 32.0 & \best{9.9} & 13.3 & 0.8 & 8.0 & 0.4 & 14.2 & 0.9 & 6.8 \\
& & GBN & 1.9 & 0.0 & 0.0 & 0.0 & 0.5 & 0.0 & 0.0 & 0.0 & 0.1 \\
& & PUAT & 0.9 & 0.6 & 0.7 & 0.5 & 0.6 & 0.5 & 0.6 & 0.5 & 0.6 \\
& & FACE & 45.5 & 0.0 & 0.2 & 0.5 & 18.1 & 0.3 & 0.5 & 0.4 & 2.9 \\
& & DAT & 27.0 & 5.0 & 15.4 & 3.5 & 13.8 & 1.6 & 12.4 & 3.2 & 7.8 \\
& & \textbf{TaFD (Ours)} & \best{47.4} & \second{7.5} & \best{21.5} & \best{16.4} & \best{29.2} & \best{6.5} & \best{23.0} & \best{15.3} & \best{17.1} \\
\bottomrule
\end{tabular*}
\end{table*}

\begin{table*}[t!]
\centering
\footnotesize
\setlength{\tabcolsep}{3.5pt}
\renewcommand{\arraystretch}{1.0}
\caption{\textbf{Generalization to Broader Heterogeneous Threats.}
Clean and robust accuracy (\%) of ResNet-34 and MobileViT-XS on CIFAR-10/100 and Tiny-ImageNet under $\ell_\infty$-APGD, $\ell_2$-APGD, ACE, StAdv, and GPGD. \textbf{Avg-RA} averages the five attacks excluding Clean. Cells in \begingroup\setlength{\fboxsep}{1.2pt}\legendbox{bestcol}{gold}\endgroup/\begingroup\setlength{\fboxsep}{1.2pt}\legendbox{secondcol}{light gold}\endgroup\ mark the column-wise \begingroup\setlength{\fboxsep}{1.2pt}\legendbox{bestcol}{best}\endgroup/\begingroup\setlength{\fboxsep}{1.2pt}\legendbox{secondcol}{second-best}\endgroup.}
\label{tab:robustness_results_avg_attackonly}
\begin{tabular*}{\linewidth}{@{\extracolsep{\fill}}llccccccc ccccccc}
\toprule
\multirow{3}{*}{\textbf{Dataset}} & \multirow{3}{*}{\textbf{Method}} & \multicolumn{7}{c}{\textbf{ResNet-34}} & \multicolumn{7}{c}{\textbf{MobileViT-XS}} \\
\cmidrule(lr){3-9} \cmidrule(lr){10-16}
& & Clean & $\ell_\infty$-APGD & $\ell_2$-APGD & ACE & StAdv & GPGD & \textbf{Avg-RA} & Clean & $\ell_\infty$-APGD & $\ell_2$-APGD & ACE & StAdv & GPGD & \textbf{Avg-RA} \\
\midrule
\multirow{9}{*}{\textbf{\textit{CIFAR-10}}} & Mean-AT & 92.1 & 19.5 & 50.7 & \second{59.8} & \second{40.6} & 69.0 & \second{47.9} & 77.0 & 39.4 & \second{63.6} & 28.9 & 19.4 & 42.5 & 38.8 \\
 & TRADES & \best{94.7} & 18.6 & 51.8 & 56.5 & 38.8 & \second{69.2} & 47.0 & \best{91.7} & 20.7 & 50.6 & \second{33.2} & \second{21.6} & \second{48.2} & 34.9 \\
 & RAMP & \second{94.0} & 19.6 & 41.6 & 42.8 & 36.4 & 60.3 & 40.1 & 68.6 & \second{40.8} & 55.8 & 2.1 & 11.1 & 38.3 & 29.6 \\
 & GBN & 83.9 & 1.8 & 8.1 & 43.0 & 34.8 & 14.5 & 20.4 & 41.5 & 0.0 & 0.1 & 2.5 & 1.3 & 2.2 & 1.2 \\
 & PUAT & 79.3 & 25.5 & \second{54.3} & 1.2 & 3.6 & 31.5 & 23.2 & 80.2 & 24.8 & 53.8 & 1.5 & 1.9 & 29.5 & 22.3 \\
 & MNG & 82.8 & \best{39.9} & \best{65.5} & 47.7 & 20.4 & 51.1 & 44.9 & 75.9 & \best{42.0} & \best{64.1} & 27.3 & 20.3 & 44.6 & \second{39.7} \\
 & FACE & 93.4 & 0.0 & 0.0 & 2.0 & 0.0 & 12.7 & 2.9 & \best{91.7} & 0.0 & 1.3 & 1.0 & 0.1 & 7.8 & 2.0 \\
 & DAT & 77.5 & 1.8 & 17.2 & 20.8 & 3.0 & 35.3 & 15.6 & 51.4 & 11.0 & 31.2 & 10.6 & 5.8 & 20.6 & 15.8 \\
 & \textbf{TaFD (Ours)} & 93.5 & \second{33.3} & 51.0 & \best{67.5} & \best{41.1} & \best{75.3} & \best{53.6} & \second{90.2} & 32.8 & 58.7 & \best{51.5} & \best{38.4} & \best{70.5} & \best{50.4} \\
\midrule
\multirow{9}{*}{\textbf{\textit{CIFAR-100}}} & Mean-AT & 68.5 & 8.4 & 27.1 & 24.2 & \second{20.5} & \second{35.3} & \second{23.1} & 51.1 & \second{20.8} & \best{37.8} & 7.2 & 10.8 & \second{20.8} & \second{19.5} \\
 & TRADES & \best{76.2} & 5.8 & 23.4 & \second{25.7} & 19.0 & 33.6 & 21.5 & \second{67.5} & 9.8 & 31.1 & \second{9.3} & 11.3 & 17.4 & 15.8 \\
 & RAMP & 55.0 & \best{24.0} & \best{38.6} & 1.3 & 5.4 & 25.2 & 18.9 & 34.9 & 18.1 & 22.8 & 0.9 & \second{15.3} & 16.9 & 14.8 \\
 & GBN & 56.4 & 1.1 & 4.5 & 11.5 & 16.0 & 5.6 & 7.7 & 2.2 & 0.0 & 0.0 & 0.2 & 0.3 & 0.4 & 0.2 \\
 & PUAT & 50.2 & 8.5 & 26.0 & 0.3 & 3.4 & 11.9 & 10.0 & 44.9 & 8.9 & 19.9 & 0.3 & 3.3 & 13.7 & 9.2 \\
 & MNG & 67.7 & 0.3 & 4.4 & 7.1 & 10.5 & 21.2 & 8.7 & 47.1 & \best{22.0} & \second{36.1} & 5.7 & 10.2 & 18.2 & 18.4 \\
 & FACE & \second{74.5} & 0.0 & 0.0 & 0.3 & 0.0 & 3.1 & 0.7 & \best{67.9} & 0.0 & 0.1 & 0.2 & 0.0 & 1.1 & 0.3 \\
 & DAT & 54.9 & 3.0 & 14.2 & 10.5 & 12.0 & 20.2 & 12.0 & 43.3 & 17.7 & 31.6 & 4.3 & 10.3 & 18.1 & 16.4 \\
 & \textbf{TaFD (Ours)} & 73.4 & \second{15.2} & \second{31.3} & \best{34.0} & \best{22.2} & \best{41.4} & \best{28.8} & 66.4 & 14.7 & 32.9 & \best{17.2} & \best{19.9} & \best{38.6} & \best{24.7} \\
\midrule
\multirow{9}{*}{\textbf{\textit{Tiny-ImageNet}}} & Mean-AT & 37.0 & 4.2 & 12.9 & \second{7.0} & \second{8.9} & \second{11.7} & \second{8.9} & 25.9 & \best{7.9} & \best{17.4} & 1.1 & \second{3.3} & \second{8.4} & \second{7.6} \\
 & TRADES & \best{50.3} & 2.7 & 14.2 & 5.8 & 2.9 & 4.6 & 6.0 & \second{44.2} & 2.9 & \second{16.3} & \second{3.2} & 3.0 & 5.7 & 6.2 \\
 & RAMP & \second{48.9} & 2.2 & 10.2 & 3.7 & 5.2 & 9.0 & 6.1 & 16.6 & 3.9 & 9.8 & 0.2 & 1.9 & 6.1 & 4.4 \\
 & GBN & 28.2 & 0.1 & 0.2 & 1.9 & 4.6 & 1.1 & 1.6 & 1.1 & 0.0 & 0.0 & 0.0 & 0.0 & 0.0 & 0.0 \\
 & PUAT & 2.0 & 0.4 & 1.1 & 0.3 & 0.7 & 0.8 & 0.7 & 0.9 & 0.6 & 0.7 & 0.5 & 0.5 & 0.6 & 0.6 \\
 & MNG & 36.2 & \second{4.4} & \second{15.3} & 3.9 & 2.9 & 10.0 & 7.3 & 23.3 & \second{7.4} & 15.8 & 1.1 & \second{3.3} & 7.5 & 7.0 \\
 & FACE & 47.3 & 0.0 & 0.0 & 0.4 & 0.0 & 0.1 & 0.1 & \best{45.5} & 0.0 & 0.2 & 0.5 & 0.0 & 0.2 & 0.2 \\
 & DAT & 16.1 & 1.9 & 6.8 & 0.6 & 2.1 & 3.7 & 3.0 & 11.6 & 3.9 & 7.5 & 0.3 & 2.7 & 3.7 & 3.6 \\
 & \textbf{TaFD (Ours)}  & 44.9 & \best{8.2} & \best{18.4} & \best{15.0} & \best{12.3} & \best{17.9} & \best{14.4} & 41.0 & 5.7 & 14.3 & \best{10.4} & \best{10.4} & \best{18.9} & \best{11.9} \\
\bottomrule
\end{tabular*}
\end{table*}

\subsection{RQ1: Robustness against Canonical Heterogeneous Threats}\label{sec:rq1}

This section evaluates the balanced robustness of TaFD under the union of canonical heterogeneous threats, comprising $\ell_p$-bounded attacks and semantic color attacks. As established in Section~\ref{sec:theory-neg-transfer-en}, these two threat classes exhibit gradient incompatibility, posing substantial challenges to JAT.
Experimental results are summarized in Table~\ref{tab:results_combined_archetypal}. TaFD achieves the highest Avg-RA in five of the six dataset--architecture settings while maintaining the highest or second-highest Clean Accuracy throughout. The advantage is most pronounced on CIFAR-100, where it attains 31.9\% Avg-RA on ResNet-34, surpassing the tied runner-up by 6.3 percentage points, and 26.6\% on MobileViT-XS, exceeding the second-best by 8.7 percentage points. On CIFAR-10, it likewise leads with 60.2\% and 53.3\% Avg-RA on the two architectures. On Tiny-ImageNet with MobileViT-XS, it achieves 17.1\% Avg-RA, outperforming the second-best by 5.6 percentage points. On Tiny-ImageNet with ResNet-34, TaFD ranks second to MNG at 12.2\% versus 14.1\%.

These results show that existing approaches cannot reconcile gradient incompatibility across heterogeneous threats. Aggregation-based JAT methods such as RAMP, MNG, and Mean-AT all compress heterogeneous gradients into a single pixel-domain update direction. When semantic color attacks join the threat union, this incompatibility intensifies sharply, directly corroborating the analysis in Section~\ref{sec:theory-neg-transfer-en}. On CIFAR-100 with ResNet-34, RAMP attains the highest $\ell_\infty$-APGD accuracy at 25.5\% yet collapses to 2.9\% on ACE. GBN introduces perturbation-specific BN branches, but this separation is confined to normalization layers and the shared convolutional parameters remain subject to the same conflict, producing an Avg-RA of only 9.7\% on CIFAR-100 with ResNet-34. PUAT models unrestricted adversarial examples via a triple-GAN, but its generative process cannot capture the structured patterns of color attacks, achieving only 1.2\% on ACE and 3.5\% on HSVAdv on CIFAR-10 with ResNet-34.
FACE and DAT both leverage spectral processing but remain threat-agnostic. FACE applies unified low-frequency downsampling and high-frequency squeezing, while DAT mixes training amplitudes with distractor images. Neither adapts its processing to the underlying threat type, resulting in Avg-RA values of only 1.6\% and 10.5\% on CIFAR-100 with ResNet-34. By contrast, TaFD conditions its spectral masks and expert routing on the diagnosed threat domain, decomposing the intractable joint objective into independently optimizable sub-problems and achieving balanced robustness where both categories of baselines fail.

\subsection{RQ2: Generalizability to Broader Heterogeneous Threats}\label{sec:rq2}

This section evaluates the generalizability of TaFD by extending the threat union to include spatial attacks (StAdv) and parametric attacks (GPGD). Compared with the $\ell_p$-color heterogeneity evaluated in RQ1, this union spans four threat dimensions, posing a more challenging multi-directional optimization conflict.
As shown in Table~\ref{tab:robustness_results_avg_attackonly}, TaFD achieves the highest Avg-RA across all six dataset--architecture settings, leading the runner-up by 4.3 to 10.7 percentage points. On CIFAR-10, TaFD attains 53.6\% and 50.4\% Avg-RA on ResNet-34 and MobileViT-XS. On CIFAR-100, it reaches 28.8\% and 24.7\%. On Tiny-ImageNet, it achieves 14.4\% and 11.9\%.

The expanded union further intensifies the robustness imbalance among baselines. On CIFAR-100 with ResNet-34, RAMP leads on $\ell_p$ attacks with 24.0\% on $\ell_\infty$-APGD and 38.6\% on $\ell_2$-APGD, but collapses to 1.3\% on ACE and 5.4\% on StAdv. MNG on the same setting shows the opposite pattern, with $\ell_\infty$-APGD at only 0.3\% and an overall Avg-RA of 8.7\%. This pronounced trade-off confirms that multi-dimensional heterogeneity makes a single pixel-domain gradient direction even less capable of reconciling competing objectives. FACE and DAT likewise fail to benefit from the expanded threat diversity, achieving Avg-RA of only 0.7\% and 12.0\% on CIFAR-100 with ResNet-34.

The uniform advantage across this broader union suggests that spectral separability extends beyond the canonical $\ell_p$-color pair. As shown in Fig.~\ref{fig:tsne_visualization}, StAdv and GPGD form well-separated clusters in the frequency domain, with signatures clearly distinct from $\ell_p$-bounded and color attacks. This separability enables effective threat-domain routing even under four-dimensional heterogeneity, providing a principled and scalable basis for defending against diverse threat unions.

\subsection{RQ3: Robustness of the Diagnosis Module}
\label{sec:rq3}

To assess the robustness of TaFD when adversaries are fully aware of the routing mechanism, we formulate an adaptive attack that jointly optimizes for misclassification and domain-classifier evasion. In a complete white-box setting with access to the true threat-domain mapping, the adversary minimizes $\mathcal{L}_{adv} = \mathcal{L}_{cls} +   \lambda_{atk} \mathcal{L}_{domain}$, where $\mathcal{L}_{cls}$ drives misclassification and $\mathcal{L}_{domain}$ forces the domain classifier to produce an incorrect threat-domain prediction. The coefficient $\lambda_{atk}$ controls the relative emphasis on evasion versus misclassification; setting $\lambda_{atk}=0$ recovers the standard attack that does not target the domain classifier. Results on CIFAR-100 with ResNet-34 across all seven attack types are reported in Fig.~\ref{fig:adaptive_radar}.

The left panel of Fig.~\ref{fig:adaptive_radar} shows that increasing $\lambda_{atk}$ progressively degrades the overall domain accuracy, confirming that adversaries can exploit spectral manipulation to evade the domain classifier. The degree of degradation varies substantially across attacks: ALA retains high domain accuracy even at $\lambda_{atk}=10$, as its perturbations are intrinsically confined to low-frequency lightness variations, whereas $\ell_\infty$-APGD and ReColorAdv drop to near-zero domain accuracy, indicating that the domain-evasion objective substantially alters their spectral characteristics. Despite the successful evasion of the domain classifier, the right panel of Fig.~\ref{fig:adaptive_radar} shows that robust accuracy does not degrade for any of the seven attacks as $\lambda_{atk}$ grows. In particular, the effect is most pronounced for attacks whose domain accuracy degrades most sharply, such as $\ell_\infty$-APGD, indicating that targeting the diagnosis module does not yield a stronger attack than the baseline attack ($\lambda_{atk}=0$) under this adaptive objective.

We attribute this result to two complementary mechanisms. The first is a self-neutralization effect: since routing decisions are governed by frequency-domain characteristics, deceiving the domain classifier compels the adversary to reshape the perturbation's spectral profile toward the target threat domain, thereby weakening its adversarial potency as the modified perturbation falls within the defense scope of the receiving domain expert. The second is an optimization trade-off: the domain-evasion objective introduces gradient conflict with the misclassification objective, resulting in a suboptimal perturbation for classification within the same perturbation budget. The latter mechanism also explains why ALA, whose domain accuracy remains largely unaffected, still exhibits improved robustness under increasing $\lambda_{atk}$: although the evasion loss $\mathcal{L}_{domain}$ fails to mislead the domain classifier, its gradient nonetheless conflicts with the misclassification gradient, producing a less effective perturbation for the classification task. Together, these two mechanisms explain why targeting the diagnosis module under this adaptive objective does not produce an attack stronger than the standard baseline.
\begin{figure}[t!]
\centering
\includegraphics[width=1\linewidth]{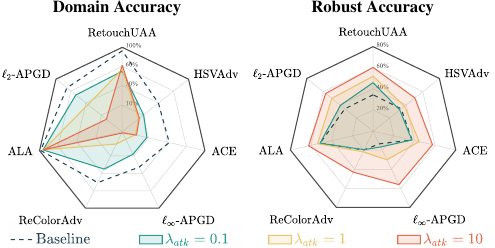}
\Description{Radar charts evaluating the robustness of TaFD under adaptive attacks targeting the domain classifier and classification.}
\caption{Robustness of TaFD under adaptive attacks jointly targeting the domain classifier and classification on CIFAR-100 with ResNet-34. \textbf{Left:} Domain accuracy per attack type. \textbf{Right:} Robust accuracy per attack type. Increasing $\lambda_{atk}$ progressively degrades domain accuracy but does not reduce robust accuracy below the standard attack ($\lambda_{atk}=0$).}
\label{fig:adaptive_radar}
\end{figure}

\subsection{RQ4: Ablation Studies and Hyperparameter Sensitivity Analysis}\label{sec:rq4}
We analyze the contribution of each TaFD component through systematic ablation and examine the sensitivity to the number of threat domains~$K$.

\subsubsection{Component Ablation}
Four variants systematically modify the Diagnosis and Dispatch stages on CIFAR-100 with ResNet-34, with results in Table~\ref{tab:ablation_studies}.

In the Diagnosis stage, Variant~1 (\emph{w/o Threat-Domain Diagnosis}) disables threat-domain conditioning entirely: the domain classifier, prototype learning, and clustering pipeline are removed, and FC-Conv processes all inputs through a single, threat-agnostic frequency-domain pathway. Avg-RA drops from 31.9\% to 26.0\%, with $\ell_\infty$-APGD falling from 13.6\% to 10.0\% and ACE from 37.4\% to 28.5\%, confirming that threat-agnostic processing cannot reconcile heterogeneous threats. Variant~2 (\emph{w/o Hungarian Alignment}) retains prototype learning and clustering but removes temporal Hungarian matching. Without stable cluster-to-domain mappings, the domain assignments for each attack type shift across epochs, causing experts to receive conflicting gradient signals that prevent effective specialization. This variant suffers the largest degradation among all ablations: Avg-RA collapses to 18.1\%, falling 7.9 percentage points below Variant~1, confirming that temporal assignment consistency is critical for stable threat-domain specialization in TaFD.

In the Dispatch stage, Variant~3 (\emph{w/o Basis-Param\-e\-ter\-ized Mask}) directly learns full-resolution spectral masks instead of parameterizing them through the low-dimensional Zernike basis. This leads to an Avg-RA decline to 30.3\%, with $\ell_2$-APGD dropping sharply from 33.8\% to 23.8\%, indicating that the Zernike basis serves as a structural regularizer that prevents overfitting in the high-dimensional spectral mask space. Variant~4 (\emph{w/o Frequency Decoupling}) removes rFFT/irFFT and spectral masking entirely, reducing FC-Conv to a spatial Mixture-of-Experts with Avg-RA of 29.4\%. In particular, color attacks degrade, e.g., HSVAdv drops from 33.4\% to 30.6\%, indicating that frequency-domain processing is necessary for separating threat-specific spectral patterns that spatial routing alone cannot distinguish.

\begin{table*}[t!]
\centering
\footnotesize
\setlength{\tabcolsep}{3.8pt}
\renewcommand{\arraystretch}{1.0}
\caption{Ablation of TaFD components on CIFAR-100 with ResNet-34.}
\label{tab:ablation_studies}
\begin{tabular}{llccccccccc}
\toprule
\textbf{Component}
  & \textbf{Method (Variant)} & Clean & $\ell_\infty$-APGD & $\ell_2$-APGD & ACE & ALA & HSVAdv & ReColorAdv & RetouchUAA & \textbf{Avg-RA} \\
\midrule
\multirow{2}{*}{Diagnosis}
& 1.\ w/o Threat-Domain Diagnosis
  & 73.0 & 10.0 & 25.8 & 28.5 & 43.2 & 25.7 & 19.9 & 29.0 & 26.0 \\
& 2.\ w/o Hungarian Alignment
  & 68.7 & 11.1 & 27.3 & 16.0 & 31.8 & 11.3 & 11.1 & 17.8 & 18.1 \\
\multirow{2}{*}{Dispatch}
& 3.\ w/o Basis-Parameterized Mask
  & 75.5 & 15.3 & 23.8 & 35.6 & 50.8 & 32.5 & 19.8 & 34.6 & 30.3 \\
& 4.\ w/o Frequency Decoupling
  & 75.0 & 15.1 & 23.6 & 34.1 & 50.2 & 30.6 & 19.8 & 32.2 & 29.4 \\
\midrule
\multicolumn{2}{l}{\textbf{TaFD (Full Model)}}
  & \textbf{76.9} & \textbf{13.6} & \textbf{33.8} & \textbf{37.4} & \textbf{51.2} & \textbf{33.4} & \textbf{19.8} & \textbf{34.2} & \textbf{31.9} \\
\bottomrule
\end{tabular}
\end{table*}

\subsubsection{Sensitivity to the Number of Threat Domains}
We evaluate the impact of the number of latent threat domains, $K$, which governs the degree of threat-aware specialization. Specifically, it determines: (i) the granularity of threat domain partitioning in the diagnosis stage, and (ii) the number of frequency band partitions and domain-expert convolutions in the dispatch stage. As illustrated in the left panel of Fig.~\ref{fig:threat_domain_k}, setting $K=1$ constrains TaFD to operate as a unified model without threat-aware specialization, producing the lowest Avg-RA of 26.0\%. Introducing decoupling substantially enhances robustness; the transition from $K=1$ to $K=2$ alone produces a 3.7-point gain, demonstrating the efficacy of threat-aware specialization. As $K$ increases further, performance generally improves with minor fluctuations and peaks at $K=6$, achieving 76.9\% clean accuracy and 31.9\% Avg-RA.
Beyond $K=6$, the trend exhibits diminishing returns: at $K=7$, Avg-RA slightly decreases to 31.5\% and clean accuracy drops to 76.2\%. This behavior indicates that excessively fine-grained partitioning elevates the difficulty for the diagnosis module, degrading the quality of the conditioning signals without providing additional specialization benefits.

Concurrently, the right panel of Fig.~\ref{fig:threat_domain_k} shows that training time scales approximately linearly with $K$, increasing from 49.0 to 81.7 minutes as $K$ grows from 1 to 7. In contrast, GPU memory overhead remains negligible, rising only from 9.85 to 9.99~GB (under 1.5\% increase), as the additional expert parameters constitute a small fraction of the total model. We select $K=6$ as the default setting, which achieves the best robustness and clean accuracy with modest computational overhead.

\begin{figure}[t]
\centering
\includegraphics[width=\linewidth]{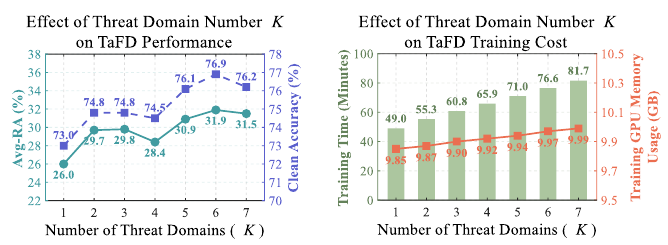}
\Description{Line graphs showing the effect of the number of threat domains on average robust accuracy and computational costs.}
\caption{Effect of the number of threat domains $K$ on TaFD performance and training cost. The left panel shows average robust accuracy (Avg-RA) and clean accuracy, while the right panel shows training GPU memory usage and training time.}
\label{fig:threat_domain_k}
\end{figure}

\section{Discussion}
\label{sec:discussion}

\noindent\textbf{Computational Overhead.}
TaFD introduces a diagnosis module and per-domain experts whose cost scales linearly with the number of threat domains $K$. As analyzed in Section~\ref{sec:rq4}, at the default setting $K{=}6$ the GPU memory overhead is less than 1.5\% (9.85$\to$9.99\,GB), and the total training time is approximately 81.7 minutes on a single GPU. This overhead is reasonable for the research setting. Future work can further reduce it through parameter sharing across experts, lightweight expert designs such as depthwise separable convolutions, and more compact FC-Conv implementations.

\noindent\textbf{Theoretical Generalizability.}
Our first-order analysis is built on subspace projection geometry and a margin constraint, and its core structure does not depend on the specific parameterization of the attack family. Although Theorem~\ref{thm:incompatibility} is instantiated with the $\ell_p$-color pair as a canonical case, the framework applies whenever a threat class can be modeled as perturbations within a subspace $V$. Empirically, the results in Section~\ref{sec:rq2} confirm that TaFD generalizes effectively to spatial and parametric threats beyond this canonical pair. Extending the formal analysis to broader threat families such as spatial deformations and generative perturbations, and investigating the feasibility of certified robustness guarantees in the frequency-decoupled setting, are promising directions.

\section{Conclusion}
\label{sec:conclusion}

This paper addresses negative transfer in joint adversarial training under heterogeneous threats. Through first-order gradient analysis, we formalize this conflict as gradient incompatibility, showing that enforcing robustness to one threat class can raise the first-order robust-risk coefficient of the other under the stated assumptions, theoretically motivating decoupled optimization. We further show that representative attacks spanning diverse threat paradigms exhibit well-separated spectral signatures in the frequency domain, establishing the frequency domain as a natural representation space for disentangling conflicting threats.

Building on these insights, we propose TaFD, a Diagnosis--Dispatch framework that reformulates JAT as a frequency-domain divide-and-conquer paradigm. Experiments on three benchmarks and two architectures demonstrate that TaFD achieves more balanced robustness than existing baselines, with an average robust accuracy gain of up to 11.3 percentage points over the strongest baseline while preserving competitive clean accuracy. This advantage generalizes to broader threat unions encompassing spatial and parametric attacks, and adaptive attacks targeting the domain classifier do not degrade robust accuracy below the standard baseline in our evaluation.

\bibliographystyle{ACM-Reference-Format}
\bibliography{main}

\appendix

\section{Open Science}
To facilitate reproducibility and support the open science initiative, we enumerate the artifacts necessary for evaluating the core contributions of this paper.
\begin{itemize}
    \item \textbf{Source Code:} The complete implementation of TaFD, including model definitions, training scripts, and all attack implementations, is available upon publication.
    \item \textbf{Evaluation Scripts:} Scripts to reproduce the training and evaluation pipeline for all experimental configurations reported in this paper are provided.
    \item \textbf{Pre-trained Models:} Pre-trained checkpoints are not included due to storage limitations. All results can be reproduced from scratch using the provided training scripts and configurations.
    \item \textbf{Datasets:} This research utilizes the publicly available CIFAR-10, CIFAR-100~\cite{krizhevsky2009learning}, and Tiny-ImageNet~\cite{le2015tiny} datasets, which are standard benchmarks and can be automatically downloaded via the provided scripts.
    \item \textbf{Availability:} The source code and evaluation scripts will be made publicly available upon publication.
\end{itemize}

\section{Ethical Considerations}
This work focuses on enhancing the robustness of deep neural networks against heterogeneous adversarial attacks, specifically bridging the gap between $\ell_p$-bounded and semantic threats.

\begin{description}
    \item[Stakeholders:] The primary stakeholders are developers, system administrators, and end-users of safety-critical computer vision systems who rely on robust visual perception in real-world deployments. A secondary group includes potential adversaries who might analyze our frequency-domain insights to design adaptive attacks.

    \item[Harms:] The detailed analysis of spectral signatures presents a theoretical dual-use risk. Advanced adversaries could leverage these frequency-domain insights to engineer perturbations that are more elusive to general detection mechanisms.

    \item[Mitigations:] However, our adaptive attack evaluation in Section~\ref{sec:rq3} empirically demonstrates that TaFD remains robust even when adversaries explicitly target the diagnosis mechanism, with robust accuracy not falling below the non-adaptive baseline.

    \item[Decision:] We conclude that the societal benefit of establishing a defense framework that can reconcile conflicting threat models outweighs the marginal risk of disclosing spectral vulnerabilities, as effective defense requires understanding these fundamental incompatibilities.
\end{description}
\end{document}